\documentclass{amsart}
\title[Gromov-Hausdorff stability of linkage-based $HC$ methods]{Gromov-Hausdorff stability of linkage-based hierarchical clustering methods}
\date{}

\author{A. Mart\'{i}nez-P\'{e}rez }
\thanks{Partially supported by MTM 2012-30719, alvaro.martinezperez@uclm.es}

\usepackage[latin1]{inputenc}
\usepackage{amsfonts}
\usepackage{amsmath}
\usepackage{mathrsfs}
\usepackage{graphicx}
\usepackage{latexsym}
\usepackage[all]{xy}
\usepackage{amsthm}
\usepackage{amssymb}

\usepackage{amsmath}
\usepackage{amscd}
\usepackage{graphics}
\usepackage{graphpap}

\newtheorem{definicion}{Definition}[section]

\newtheorem{prop}[definicion]{Proposition}
\newtheorem{lema}[definicion]{Lemma}
\newtheorem{obs}[definicion]{Remark}
\newtheorem{teorema}[definicion]{Theorem}
\newtheorem{cor}[definicion]{Corollary}
\newtheorem{ejp}[definicion]{Example}

\newcommand{\co}{\ensuremath{\colon}} 
\newcommand{\bn}{\ensuremath{\mathbb{N}}} 
\newcommand{\br}{\ensuremath{\mathbb{R}}} 

\begin{document}

\begin{abstract} A hierarchical clustering method is stable if small perturbations on the data set produce small perturbations in the result. This perturbations are measured using the Gromov-Hausdorff metric. We study the problem of stability on linkage-based hierarchical clustering methods. We obtain that, under some basic conditions, standard linkage-based methods are semi-stable. This means that they are stable if the input data is close enough to an ultrametric space. We prove that, apart from exotic examples, introducing any unchaining condition in the algorithm always produces unstable methods.
\end{abstract}

\maketitle

\begin{footnotesize}
\end{footnotesize}

\tableofcontents

\section{Introduction}

Clustering techniques are extensively used in data analysis. Given a finite metric space, a clustering method yields a partition of the space. The output of a hierarchical clustering algorithm can be described as a nested family of partitions, as a dendrogram or as an ultrametric space. These procedures are used also in the more general setting where the input is a finite set endowed with a distance function. See \cite{AB_11}. In our work, the input is always a finite metric space.

One common type of clustering algorithms are those which are determined by a linkage function. These algorithms are called \emph{linkage-based} methods and a characterization can be found in \cite{AB_11}. Linkage-based hierarchical clustering algorithms start with singleton clusters. Then, in the recursive formulation, the linkage function is used at every step to determine the minimal distance between clusters, $R_i$. Then, clusters at distance $R_i$ are merged. The most usual methods of hierarchical clustering such as single linkage, complete linkage and average linkage are linkage-based. Herein, we refer to this type of methods as \textit{standard linkage-based methods}.

In \cite{M-P2} we studied some characteristic chaining effect in hierarchical clustering. For further reference about the chaining effect see also  \cite{LW}. To deal with it, different strategies have been proposed, see \cite{CM2}, \cite{EKSX} or \cite{W}. We defined an algorithm called $\alpha$-unchaining single linkage or $SL(\alpha)$ to avoid certain type of chaining effect. The key was to include some property, $P_\alpha$, so that if a pair blocks does not satisfy it, the blocks are not merged.  In this case, $P_\alpha$ is defined so that if the blocks are \textit{chained by a single edge} (which we may see as an undesired chaining) then they will not satisfy the property and they stay as separated blocks until the next step in the recursive formulation. 

To treat other undesired chaining effects in the application of a linkage-based algorithm the same strategy can be used. First, characterizing when a pair of blocks should not be merged and then defining some unchaining condition $P$ to prevent it. 
This can be generalized as well considering hierarchical clustering algorithms determined by some linkage function, $\ell$, and some \textit{unchaining condition} $P$. We call these algorithms \emph{almost-standard linkage-based methods}.

An important property for a hierarchical clustering method is some kind of stability under small perturbations in the input data. Otherwise, small differences may yield very different dendrograms and the significance of the hierarchical clustering can be compromised. This can be a major problem if the input is a random sample determined by some probability distribution or if the input data is only known up to a certain degree of precision. There are several papers dealing with stability of clustering and hierarchical clustering.  Different approaches and different types of algorithms are considered.  In \cite{BL} there is an interesting analysis on stability for clustering algorithms based on some objective function which is minimized or maximized. Other perspectives can be found in \cite{BE}, \cite{KN} and \cite{RC}. In our work we focus on linkage-based algorithms. In \cite{CM}, Carlson and Memoli use tools from geometric topology to study the stability of single linkage hierarchical clustering. In their paper  
stability is measured using Gromov-Hausdorff metric, $d_{GH}$, and the following result is obtained.

\begin{prop}\cite[Proposition 26]{CM} \label{Stable SL} 
For any two finite metric spaces $(X, d_X)$ and $(Y, d_Y)$
\[d_{\mathcal{GH}}(X,Y)\geq d_{\mathcal{GH}}(\mathfrak{T}^{SL}_\mathcal{U}(X,d_X),\mathfrak{T}^{SL}_\mathcal{U}(Y,d_Y)).\] 
\end{prop}

Thus, if single linkage hierarchical clustering is applied to a pair of metric spaces which are close in the Gromov-Hausdorff metric, then the resultant ultrametric spaces are also close in this metric. Moreover, the Gromov-Hausdorff metric has the advantage that it can be used to measure the distance between any pair of metric spaces, regardless of whether they have the same number of points or not. 

This work follows the same path. We give explicit definitions to study the stability of hierarchical clustering methods in terms of the Gromov-Hausdorff metric. First we consider the notion of being \textit{semi-stable in the Gromov-Hausdorff sense} for algorithms that are stable at least when the input data is close enough to an ultrametric space.  Then, we define the stronger property of being 
\textit{stable in the Gromov-Hausdorff} sense. This definition contains the implicit approach given by Carlson and Memoli so that the condition given in Proposition \ref{Stable SL} implies that $SL$ is stable in the Gromov-Hausdorff sense.

Our purpose is to analyse the problem of Gromov-Hausdorff stability in standard and almost-standard linkage-based hierarchical clustering methods. 
We pay special attention to single linkage, complete linkage, average linkage and $\alpha$-unchaining single linkage. However, we try to generalize our results working with arbitrary linkage functions and arbitrary unchaining conditions. Since these definitions are quite unrestricted, we must include several technical properties to avoid exotic cases and make proofs work. 
Nevertheless, these technical  properties are trivially satisfied by the usual hierarchical clustering methods mentioned above and the theorems obtained are, in fact, very general.

We prove that every normal faithful standard linkage-based method is semi-stable. See Theorem \ref{Th: semistable}. Therefore, if we know that the data correspond to an ultrametric space and the measures are precise, semi-stability can be good enough to obtain legitimate results and standard linkage-based methods would be adequate. 

Nevertheless, in Proposition \ref{Prop: main_1} we prove that any unchaining condition introduces some degree of instability in the method. For every admissible almost-standard linkage-based method, there exist two arbitrarily close metric spaces with a pair of blocks each such that for the first pair the unchaining condition is satisfied and for the other one is not.  This instability can be stated more precisely if the linkage function is $\ell^{SL}$. In this case, if the  almost-standard linkage-based method is ordinary and the unchaining condition is minimally bridge-unchaining, then the method is not stable in the Gromov-Hausdorff sense. See Theorem \ref{Teorema: stable}. In particular, $SL(\alpha)$ is not stable in the Gromov-Hausdorff sense. 

The structure of the paper is the following:

Section \ref{Section: background} recalls the basics on hierarchical clustering and the main notations used herein. 

Section \ref{Section: HC} presents the recursive definition of almost-standard linkage-based hierarchical clustering methods.

In section \ref{Section: semistable} we define the property of being \emph{semi-stable in the Gromov-Hausdorff sense}. We prove that any normal, faithful, standard linkage-based $HC$ method is semi-stable in the Gromov-Hausdorff sense. This includes the classical standard linkage-based $HC$ methods. We also prove that $SL(\alpha)$ is semi-stable in the Gromov-Hausdorff sense.

In section \ref{Section: stable} we define the property of being \emph{stable in the Gromov-Hausdorff sense}.  We prove that for every ordinary almost-standard linkage-based $HC$ method, $\mathfrak{T}(\ell^{SL},P)$, if $P$ is minimally bridge-unchaining, then 
$\mathfrak{T}$ is not stable in the Gromov-Hausdorff sense. In particular, $SL(\alpha)$ is not stable in the Gromov-Hausdorff sense.

Section \ref{Section: Conclusions} includes a short discussion about Gromov-Hausdorff stability.

\section{Background and notation}\label{Section: background}

A dendrogram over a finite set is a nested family of partitions. This is usually represented as a rooted tree. 

Let $\mathcal{P}(X)$ denote the collection of all partitions of a finite set $X=\{x_1,...,x_n\}$. Then, a dendrogram can also be described as a map $\theta\co [0,\infty)\to \mathcal{P}(X)$ such that:
\begin{itemize}
	\item[1.] $\theta(0)=\{\{x_1\},\{x_2\},...,\{x_n\}\}$,
	\item[2.] there exists $T$ such that $\theta(t)=X$ for every $t\geq T$,
	\item[3.] if $r\leq s$ then $\theta(r)$ refines $\theta(s)$,
	\item[4.] for all $r$ there exists $\varepsilon >0$ such that $\theta(r)=\theta(t)$ for $t\in [r,r+\varepsilon]$.
\end{itemize}

Notice that conditions 2 and 4 imply that there exist $t_0<t_1<...<t_m$ such that $\theta(r)=\theta(t_{i-1})$ for every $r\in [t_{i-1},t_i)$, $i=1,m$ and $\theta(r)=\theta(t_{m})=\{X\}$ for every $r\in [t_m,\infty)$.

For any partition $\{B_1,...,B_k\}\in \mathcal{P}(X)$, the subsets $B_i$ are called \textit{blocks}.

Let $\mathcal{D}(X)$ denote the collection of all possible dendrograms over a finite set $X$. Given some $\theta \in \mathcal{D}(X)$, let us denote $\theta(t)=\{B_1^t,...,B_{k(t)}^t\}$. Therefore, the nested family of partitions is given by the corresponding partitions at $t_0,...,t_m$, this is, $\{B_1^{t_i},...,B_{k(t_i)}^{t_i}\}$ $i=0,m$.

An \emph{ultrametric space} is a metric space $(X,d)$ such that 
$d(x,y)\leq \max \{d(x,z),d(z,y)\}$
for all $x,y,z\in X$. Given a finite metric space $X$ let $\mathcal{U}(X)$ denote the set of all ultrametrics over $X$.

There is a well known equivalence between trees and ultrametrics. See \cite{Hug} or \cite{M-M} for a complete exposition of how to build categorical equivalences between them. In particular, this may be translated into an equivalence between dendrograms and ultrametrics:

Thus, a hierarchical clustering method $\mathfrak{T}$ can be presented as an algorithm whose output is a a dendrogram or an ultrametric space.  Let  $\mathfrak{T}_{\mathcal{D}}(X,d)$ denote the dendrogram obtained by applying $\mathfrak{T}$ to a metric space $(X,d)$ and $\mathfrak{T}_{\mathcal{U}}(X,d)$ denote the corresponding ultrametric space.

Let us define the map $\eta \co \mathcal{D}(X) \to \mathcal{U}(X)$ as follows:

Given a dendrogram $\theta\in \mathcal{D}(X)$, let $\eta(\theta)=u_\theta$ be such that $u_\theta(x,x')=\min\{r\geq 0 \, | \, x,x' \mbox{ belong to the same block of } \theta(r)\}$.

\begin{prop}\cite[Theorem 9]{CM} $\eta$ is a bijection such that $\eta \circ \mathfrak{T}_{\mathcal{D}}=\mathfrak{T}_{\mathcal{U}}$.
\end{prop}

\textbf{Notation}: For any $HC$ method $\mathfrak{T}$ and any finite metric space $(X,d)$, let us denote $\mathfrak{T}_{\mathcal{D}}(X,d)=\theta_X$ and $\mathfrak{T}_{\mathcal{U}}(X,d)=(X,u_X)$. If there is no need to distinguish the metric space we shall just write $\mathfrak{T}_{\mathcal{D}}(X,d)=\theta$ and $\mathfrak{T}_{\mathcal{U}}=u$.

\section{Linkage-based hierarchical clustering methods}\label{Section: HC}

Let us recall the definition of some hierarchical clustering methods. We include here the description of single linkage by its $t$-connected components and the recursive description of single linkage, complete linkage and average linkage. We recall also the alternative description of these methods, defining the graph $G_R^\ell$, which we used to build our method, $SL(\alpha)$. This process might be useful to define other algorithms better adapted to specific problems. 

An $\varepsilon$-\emph{chain} is a finite sequence of points $x_0, ..., x_N$ that are separated by distances less or equal than $\varepsilon$: $|x_i - x_{i+1}| < \varepsilon$. Two points are $\varepsilon$-\emph{connected} if there is an $\varepsilon$-chain joining them. Any two points in an $\varepsilon$-\emph{connected set} can be linked by an $\varepsilon$-chain. An $\varepsilon$-\emph{component} is a maximal $\varepsilon$-connected subset.

Clearly, given a metric space and any $\varepsilon>0$, there is a partition of $X$ in its $\varepsilon$-components $\{C_1^\varepsilon,...,C_{k(\varepsilon)}^\varepsilon\}$. 

Let $X$ be a finite metric set. The \textit{single linkage} $HC$ is defined by the map $\theta\co [0,\infty)\to \mathcal{P}(X)$ such that $\theta(t)$ is the partition of $X$ in its $t$-components.

Let us recall the following 
from \cite{ABL}.

For $x, y \, \in  \, X$ and any (standard) clustering $C$ of $X$, $x \sim_C y$ if $x$ and $y$ belong to the same cluster in $C$ and $x \not \sim_C y$, otherwise.

Two (standard) clusterings  $C = (C_1,...,C_k)$ of $(X, d)$ and $C' = (C'_1,...C'_k)$ 
of $(X',d')$ are isomorphic clusterings, denoted $(C,d) \cong (C',d')$, if there exists a bijection
$\phi : X \to X'$ such that for all $x, y \, \in  \, X$, $d(x, y) = d'(\phi(x),\phi(y))$ and $x \sim_C y$ if and
only if $\phi(x) \sim_{C'} \phi(y)$.

\begin{definicion} A \textit{linkage function} is a function 
\[\ell : \{(X_1,X_2, d) \, | \, d \mbox{ is a distance function over } X_1 \cup X_2 \} \to R^+ \]
such that,
\begin{itemize}
\item[1.] $\ell$ is \textit{representation independent}: For all $(X_1,X_2)$ and $(X'_1,X'_2)$, if $(X_1,X_2, d) \cong (X'_1,X'_2, d')$ (i.e.,
they are clustering-isomorphic), then $\ell(X_1,X_2, d) = \ell(X'_1,X'_2, d')$.
\item[2.] $\ell$ is \textit{monotonic}: For all $(X_1,X_2)$ if $d'$ is a distance function over $X_1\cup X_2$ such that for all $x \sim_{\{X_1,X_2\}}y$, $d(x, y) = d'(x, y)$ and for all $x \not \sim_{\{X_1,X_2\}}y$, $d(x, y) \leq d'(x, y)$ then $\ell(X_1,X_2, d')\geq \ell(X_1,X_2, d) $.
\item[3.] Any pair of clusters can be made arbitrarily distant: For any pair of data sets $(X_1, d_1)$, $(X_2, d_2)$, and
any $r$ in the range of $\ell$, there exists a distance function $d$ that extends $d_1$ and $d_2$ such that $\ell(X_1,X_2, d) >r$.
\end{itemize}
\end{definicion}

For technical reasons, let us assume that a linkage function has a countable range. Say, the set of nonnegative
algebraic real numbers.

Some standard choices for $\ell$ are:

\begin{itemize}
	\item Single linkage: $\ell^{SL}(B,B')=\min_{(x,x')\in B\times B'}d(x,x')$ 
	\item Complete linkage: $\ell^{CL}(B,B')=\max_{(x,x')\in B\times B'}d(x,x')$
	\item Average linkage: $\ell^{AL}(B,B')=\frac{\sum_{(x,x')\in B\times B'}d(x,x')}{\#(B)\cdot \#(B')}$ where $\#(X)$ denotes the cardinal of the set $X$.
\end{itemize}

In \cite{CM}, the authors  use a recursive procedure by which they redefine $SL$ $HC$, average linkage ($AL$) and complete 
linkage ($CL$) hierarchical clustering. We reproduce here, for completeness, their formulation. 

Let $(X,d)$ be a finite metric space where $X=\{x_1,...,x_n\}$ and let $L$ denote a family of linkage functions on $X$:
\[L:=\{\ell \colon \mathcal{C}(X)\times \mathcal{C}(X) \to \br^+ \, | \, \ell \mbox{ is bounded and non-negative } \}\]
where $\mathcal{C}(X)$ denotes the collection of all non-empty subsets of $X$.

Fix some linkage function $\ell\in L$. Then, the recursive formulation is as follows

\begin{itemize}
	\item[1.] For each $R>0$ consider the equivalence relation $\sim_{\ell,R}$ on blocks of a partition $\Pi\in \mathcal{P}(X)$, given by $B\sim_{\ell,R}B'$ if and only if there is a sequence of blocks $B=B_1,...,B_s=B'$ in $\Pi$ with $\ell(B_k,B_{k+1})\leq R$ for $k=1,...,s-1$.
	\item[2.] Consider the sequences $R_0,R_1,R_2,... \in [0,\infty)$ and $\Theta_1, \Theta_2,... \in \mathcal{P}(X)$ given by  $R_0=0$, $\Theta_1:=\{x_1,...,x_n\}$, and recursively for $i\geq 1$ by $\Theta_{i+1}=\frac{\Theta_i}{\sim_{\ell,R_i}}$ where $R_i:=\min\{\ell(B,B')\, | \, B,B'\in \Theta_i, \ B\neq B'\}$ until $\Theta_{i}=\{X\}$.
	\item[3.] Finally, let $\theta^\ell\colon [0,\infty)\to \mathcal{P}(X)$ be such that $\theta^\ell(r):=\Theta_{i(r)}$ with $i(r):=\max\{i\, | \, R_i\leq r\}$.
\end{itemize}

In \cite{M-P2} we proposed the following formulation of the recursive algorithm:

\begin{definicion} We say that $\mathfrak{T}$  is a \textbf{standard linkage-based hierarchical clustering method} if there is some linkage function $\ell$ such that $\mathfrak{T}_\mathcal{D}$ can be defined recursively as follows:

\begin{itemize}
	\item[1.] Let $\Theta_1:=\{x_1,...,x_n\}$ and $R_0=0$.
	\item[2.] For every $i\geq 1$, while $\Theta_{i}\neq \{X\}$, let $R_i:=\min\{\ell(B,B')\, | \, B,B'\in \Theta_i, \ B\neq B'\}$. Then, let $G_{R_i}^\ell$ be a graph whose vertices are the blocks of $\Theta_i$ and such that there is an edge joining $B$ and $B'$ if and only if $\ell(B,B')\leq R_i$. 
	\item[3.] Consider the equivalence relation $B\sim_{\ell,R} B'$ if and only if $B,B'$ are in the same connected component of $G_R^\ell$. Then, $\Theta_{i+1}=\frac{\Theta_i}{\sim_{\ell,R}}$.
	\item[4.] Finally, let $\theta^\ell\colon [0,\infty)\to \mathcal{P}(X)$ be such that $\theta^\ell(r):=\Theta_{i(r)}$ with $i(r):=\max\{i\, | \, R_i\leq r\}$.
\end{itemize} 
\end{definicion}

\begin{prop} $\mathfrak{T}_\mathcal{D}$ is a dendrogram.
\end{prop}

\begin{proof} It is clear by construction that $\Theta_i$ refines $\Theta_{i+1}$ for every $i$. Now, if $r<r'$, $i(r)=\max\{i\, | \, R_i\leq r\}$ and $i(r')=\max\{i\, | \, R_i\leq r'\}$, then $i(r)\leq i(r')$. Hence, it is immediate to check that $\theta^{\ell}$ is a dendrogram.
\end{proof}

Notice, however, that the sequence $(R_i)$ defined in the algorithm above need not be increasing:

\begin{ejp} Let $\ell$ be a linkage function defined as follows: $$\ell(B_1,B_2):=\frac{\ell^{SL}(B_1,B_2)}{\sharp(B_1)+\sharp(B_2)}.$$
It is immediate to check that $\ell$ is a well defined linkage function.

Let  $X=\{0,2,5,7\}\subset \br$ with the euclidean metric $d$ inherited. If we apply the recursive algorithm above, then $R_1=\ell(0,2)=\ell(5,7)=1$ and $\Theta_1=\{\{0,2\},\{5,7\}\}$. However, $R_2=\ell(\{0,2\},\{5,7\})=\frac{3}{4}<R_1$.

Notice that $\theta^\ell(r)=\{\{0\},\{2\},\{5\},\{7\}\}$ for every $r<\frac{3}{4}$ and $\theta^\ell(r)=\{0,2,5,7\}$ for every $r\geq \frac{3}{4}$.
\end{ejp}

\begin{definicion} We say that a linkage function is \textbf{increasing} if the sequence $(R_i)$ defined in the recursive construction is increasing.
\end{definicion}

\begin{prop}\label{Prop: increasing} $\ell^{SL}$, $\ell^{CL}$ and $\ell^{AL}$ are increasing linkage functions. 
\end{prop}

\begin{proof} Let $\ell$ be $\ell^{SL}$, $\ell^{CL}$ or $\ell^{AL}$. First, notice that for any pair of blocks $A_1,A_2$, $\ell(A_1,A_2)$ only depends of the distance between the points in $A_1$ and $A_2$. 

Claim: Given $R_i$ and $\Theta_i$, then $R_{i+1}>R_i$. 

This is immediate for $\ell^{SL}$. 

Let us check the proposition for $\ell=\ell^{CL}$. Let $B_1,B_2\in \Theta_{i+1}$ and $A_1,A_2\in \Theta_i$ such that $A_1\subset B_1$ and $A_2\subset B_2$. Then, 
$\ell^{CL}(A_1,A_2)\leq \ell^{CL}(A_1,B_2) \leq \ell^{CL}(B_1,B_2)$. Since $A_1,A_2$ are not merged in $\Theta_{i+1}$,  $R_i<\ell^{CL}(A_1,A_2)\leq \ell^{CL}(B_1,B_2)$. Hence, $R_{i+1}>R_i$.

Now, consider $\ell=\ell^{AL}$. 

First notice that if $\{a_1,...,a_n\}\subset \br$ and $\{b_1,...,b_n\}\subset \br$ with 
$$M_1=\frac{\sum_1^n a_i}{n}> R \mbox{ and } M_2=\frac{\sum_1^m b_i}{m}> R$$ then 
\begin{equation}\label{Eq: media} M_3=\frac{\sum_1^{n} a_i+\sum_1^{m} b_i}{n+m}> R.\end{equation}
It suffices to check that $(n+m)M_3=nM_1+mM_2\geq (n+m)M_3$ and, therefore, $M_3> R$.

Let $B_1,B_2\in \Theta_{i+1}$, $A_1,A_2,...,A_k\in \Theta_i$ such that $A_1\cup \cdots \cup A_k= B_1$ and $C_1,...,C_{k'}\in \Theta_i$ with $C_1\cup \cdots \cup C_{k'}=  B_2$. Clearly, $\ell^{AL}(A_i,C_j)> R_i$ since they are independent blocks in $\Theta_i$. Thus, applying (\ref{Eq: media}), $\ell^{AL}(A_i,B_2)> R_i$ and  $\ell^{AL}(B_1,B_2)>R_i$. Hence, $R_{i+1}>R_i$.
\end{proof}

Notation: Given a metric space $(X,d)$ let \begin{equation}\label{Eq: delta} \Delta(X,d):=\{d(x,y)\, | \, x,y \in \, X\}\end{equation} this is, the set of all possible distances between points in $(X,d)$.

\begin{definicion} Let $P$ be any condition on the set of pairs of blocks which depends on a parameter $R$. We say that $P$ is an \textbf{unchaining condition} if for every pair of blocks $B,B'$ there is some $R_0>0$ such that $\{B_1,B_2\}$ satisfies $P$ for every $R>R_0$.
\end{definicion}

\begin{definicion}  We say that $\mathfrak{T}$ is an \textbf{almost-standard linkage-based} hierarchical clustering method if there is an unchaining condition $P$ such that $\mathfrak{T}_\mathcal{D}$ can be defined as follows.

\begin{itemize}
	\item[1.] Let $\Theta_1:=\{x_1,...,x_n\}$, $\Theta_0=\emptyset$ and $R_0=0$.
	\item[2.] For every $i\geq 1$, while $\Theta_{i-1}\neq \{X\}$:
	
\begin{itemize}
	\item[a)] if $\Theta_{i}\neq \Theta_{i-1}$ let $R_i:=\min\{\ell(B,B')\, | \, B,B'\in \Theta_i, \ B\neq B'\}$,
	\item[b)] if $\Theta_{i}= \Theta_{i-1}$ and there exist $B,B'\in \Theta_i$ such that $\ell(B,B')>R_{i-1}$, let $R_i:=\min\{\ell(B,B')\, | \, B,B'\in \Theta_i \mbox{ and } \ell(B,B')>R_{i-1}\}$,
	\item[c)] if $\Theta_{i}= \Theta_{i-1}$ and there are no blocks $B,B'\in \Theta_i$ such that $\ell(B,B')>R_{i-1}$, then let $R_i:=\min\{R \, | \, \mbox{ there exist } B,B'\in \Theta_i \mbox{ such that } \{B,B'\}\\ \mbox{ satisfies $P$ for R}\}$.
\end{itemize}

	Then, let $G_{R_i}^\ell$ be a graph whose vertices are the blocks of $\Theta_i$ and such that there is an edge joining two blocks, $B$ and $B'$, if and only if the following two conditions hold.

\begin{itemize}
	\item[i)] $\ell(B,B')\leq R_i$ 
	\item[ii)] $\{B,B'\}$ satisfies condition $P$ for $R_i$.
\end{itemize}

	\item[3.] Consider the equivalence relation $B\sim_{\ell,R_i} B'$ if and only if $B,B'$ are in the same connected component of $G_{R_i}^\ell$. Then, $\Theta_{i+1}=\frac{\Theta_i}{\sim_{\ell,R_i}}$.
	\item[4.] Finally, let $\theta^\ell\colon [0,\infty)\to \mathcal{P}(X)$ be such that $\theta^\ell(r):=\Theta_{i(r)}$ with $i(r):=\max\{i\, | \, R_i\leq r\}$.
\end{itemize}

In this case, we say that $P$ is the \emph{unchaining condition} of the method. Hence, an almost-standard linkage-based HC method is defined by a linkage function $\ell$, some property $P$. Let us denote $\mathfrak{T}=\mathfrak{T}(\ell;P)$.
\end{definicion}

\begin{obs}\label{Obs: ordered} Notice that step 2 has to deal with some technical difficulties which do no show up in the case of standard linkage-based algorithms. Without conditions $b)$ and $c)$, the algorithm may not finish. However, condition $c)$ may be a problem for practical uses. 

An alternative procedure to assure that the algorithm will not enter into an infinite loop is to begin by choosing an ordered set containing all possible distances between blocks, $\mathcal{R}=\{R_0,...,R_m\}$ with $R_i>R_{i-1}$, making sure that $R_m$ is big enough so that $\theta(R_m)=X$. If the linkage function is $\ell^{SL}$ or $\ell^{CL}$ this may be easily solved considering $\mathcal{R}=\{R_0,...,R_m\}$ to be the ordered set of distances between the points in the sample $(X,d)$, i.e. $\mathcal{R}=\Delta(X,d)$. In this case, all possible distances between blocks are contained in $\Delta(X,d)$ and the resulting dendrogram will be the same as applying the formal algorithm above. Nevertheless, we believe that for the general statement of almost-standard linkange-based methods is better not to include any election \textit{a priori} of the sequence $(R_i)$.
\end{obs}

\begin{ejp} Notice that $\mathfrak{T}_{SL}=\mathfrak{T}(\ell^{SL},\emptyset)$.
\end{ejp}

Let us recall here the definition of $SL(\alpha)$ from \cite{M-P2}:

Given a finite metric space $(X,d)$, let $F_t(X,d)$ be the Rips (or Vietoris-Rips) complex of $(X,d)$. Let us recall that the Rips complex of a metric space $(X,d)$ is a simplicial complex whose vertices are the points of $X$ and $[v_1,...,v_k]$ is a simplex of $F_t(X,d)$ if and only if $d(v_i,v_j)\leq t$ for every $i,j$. Given any subset $Y\subset X$, by $F_t(Y)$ we refer to the subcomplex of $F_t(X)$ defined by the vertices in $Y$.

$\mathcal{R}=\{R_0,...,R_m\}=\Delta(X,d)$ with $R_i>R_{i-1}$. Clearly, $R_0=0$. 

Let the dendrogram defined by $SL(\alpha)$, $\theta_\alpha$, be as follows:

\begin{itemize}

	\item[1)] Let $\theta_\alpha(0):=\{\{x_1\},...,\{x_n\}\}$ and $\theta_\alpha(r):=\theta_\alpha(0)$ $\forall r< R_1$. Now, given $\theta_\alpha[R_{i-1},R_i)=\theta_\alpha (R_{i-1})=\{B_1,...,B_m\}$, we define recursively  $\theta_\alpha$ on the interval 
	$[R_i,R_{i+1})$ as follows: 

	\item[2)] Let $G_\alpha^{R_i}$ be a graph with vertices $\mathcal{V}(G_\alpha^{R_i}):=\{B_1,...,B_m\}$ and edges $\mathcal{E}(G_\alpha^{R_i}):=\{B_i,B_j\}$ such that the following conditions hold:
		\begin{itemize}
			\item[i)] $\min\{d(x,y)\, | \, x\in B_i,\ y\in B_j\}\leq R_i$.
			\item[ii)] there is a simplex $\Delta \in F_{R_i}(B_i\cup  B_j)$  such that $\Delta \cap B_i\neq \emptyset$, $\Delta \cap B_j\neq \emptyset$ and $\alpha \cdot dim(\Delta)\geq \min\{dim (F_{R_i}(B_i)), dim (F_{R_i}(B_j))\}$. 
		\end{itemize}

	\item[3)] Let us define a relation, $\sim_{R_i,\alpha}$ as follows.
		
Let $B_i\sim_{R_i,\alpha}B_j$ if $B_i,B_j$ belong to the same connected component of the graph $G_\alpha^{R_i}$. Then, $\sim_{R_i,\alpha}$ induces an equivalence relation.

	\item[4)] For every  $r\in [R_i,R_{i+1})$, $\theta_\alpha(r):=\theta_\alpha(R_{i-1})/\sim_{R_i,\alpha}$.
\end{itemize}

Here, the unchaining condition $ii)$ is defined to avoid the chaining effect which automatically merges two blocks when the minimal distance between them is small. See \cite{M-P2}.

\begin{ejp} $SL(\alpha)$ is an almost-standard linkage-based hierarchical clustering method: 

$SL(\alpha)=(\ell^{SL},P_\alpha)$ where given $B_j,B_k\in \Theta_i$, $\{B_j,B_k\}$ satisfies $P_\alpha$ for $R_i$ if there is a simplex $\Delta \in F_{R_i}(B_j\cup  B_k)$  such that $\Delta \cap B_j\neq \emptyset$, $\Delta \cap B_k\neq \emptyset$ and $\alpha \cdot dim(\Delta)\geq \min\{dim (F_{R_i}(B_j)), dim (F_{R_i}(B_k))\}$.
\end{ejp}

\section{Semi-stability of HC methods.}\label{Section: semistable}

Let us recall the definition of Gromov-Hausdorff distance from \cite{BBI}. See also \cite{Gr}.

Let $(X,d_X)$ and $(Y,d_Y)$ be two metric spaces. A correspondence (between $A$ and $B$) is a subset $ \tau\in A\times B$ such that
\begin{itemize}
	\item $\forall \, a\in A$, there exists $b\in B$ s.t. $(a,b)\in \tau$
	\item $\forall \, b\in B$, there exists $a\in A$ s.t. $(a,b)\in \tau$
\end{itemize}

Let $\mathcal{T}(A,B)$ denote the set of all possible correspondences between $A$ and $B$.

Let $\Gamma_{X,Y}\colon X\times Y \times X \times Y\to \br^+$ given by \[(x,y,x',y')\mapsto |d_X(x,x')-d_Y(y,y')|.\]

Then, the \textit{Gromov-Hausdorff distance} between $X$ and $Y$  is:

\[d_{\mathcal{GH}}(X,Y):=\frac{1}{2} \inf_{\tau \in \mathcal{T}(X,Y)} \sup_{(x,y)(x',y')\in \tau}\Gamma_{X,Y}(x,y,x',y').\]

\textbf{Notation:} Let $(\mathcal{M},d_{GH})$ denote the set of finite metric spaces with the Gromov-Hausdorff metric and $(\mathcal{U},d_{GH})$ denote the set of finite ultrametric spaces with the Gromov-Hausdorff metric.

\begin{definicion} A $HC$ method $\mathfrak{T}$ is \textbf{faithful}  if for any ultrametric space $(U,d)$, $\mathfrak{T}_\mathcal{U}(U,d)=(U,d)$.
\end{definicion}

The following appears as Proposition 4.6 in \cite{M-P2}. We include the proof for completeness.

\begin{prop}\label{Prop: faithful SL} $\mathfrak{T}^{SL}$ is faithful. 
\end{prop}

\begin{proof} By definition, it is clear that $u_{SL}(x,y)\leq d(x,y)$ for every $x,y\in X$.

Let us see that, if $(X,d)$ is an ultrametric space, then $u_{SL}(x,y)\geq d(x,y)$. $u_{SL}(x,y)=\inf\{t \, | \, \mbox{there exists a $t$-chain joining } x \mbox{ to } y\}$. Suppose $u_{SL}(x,y)=t$ and let $x=x_0,x_1,...,x_n=y$ a $t$-chain joining $x$ to $y$. By the properties of the ultrametric, $d(x_{i-1},x_{i+1})\leq \max\{d(x_{i-1},x_{i}),d(x_{i},x_{i+1})\}\leq t$ for every $1\leq i\leq n-1$. Therefore, $d(x,y)\leq t$ and $u_{SL}(x,y)\geq d(x,y)$.
\end{proof}

\begin{prop} $\mathfrak{T}^{CL}$ and $\mathfrak{T}^{AL}$ are faithful.
\end{prop}

\begin{proof} Let $(X,u)$ be a finite ultrametric space and $\Delta(X,d)=\{0=t_0<t_1<\cdots <t_m\}$.

Let us check that  $\mathfrak{T}^{SL}_{\mathcal{D}}(U,d)=\mathfrak{T}^{CL}_{\mathcal{D}}(U,d)=\mathfrak{T}^{AL}_{\mathcal{D}}(U,d)$. Therefore, by Proposition \ref{Prop: faithful SL}, $\mathfrak{T}^{CL}_{\mathcal{U}}(U,d)=\mathfrak{T}^{AL}_{\mathcal{U}}(U,d)=(U,d)$.

Let $\mathfrak{T}^{SL}_{\mathcal{D}}(U,d)=\theta_{SL}$, $\mathfrak{T}^{CL}_{\mathcal{D}}(U,d)=\theta_{CL}$ and $\mathfrak{T}^{CL}_{\mathcal{D}}(U,d)=\theta_{CL}$.

For $t_0=0$, every point of $U$ is a block in  $\theta_{CL}(t_0)$ or $\theta_{AL}(t_0)$ and  $\theta_{CL}(t_0)=\theta_{AL}(t_0)=\theta_{SL}(t_0)$.

Now, suppose that $\theta_{CL}(t_{i-1})=\theta_{AL}(t_{i-1})=\theta_{SL}(t_{i-1})$. Then, the blocks are exactly the balls of radius $t_{i-1}$.
Let $B,B'\in \theta_{SL}(t_{i-1})$. Let $x_0\in B$ and $y_0\in B'$ be such that $d(x_0,y_0)=\min_{(x,x')\in B\times B'}d(x,x')$. Then, $d(x_0,y_0)>t_{i-1}$ and, by the properties of the ultrametric, for every pair of points $(x,x')\in B\times B'$, $d(x,x')\leq \max\{t_{i-1},d(x_0,y_0)\}=d(x_0,y_0)$. Thus, $\ell^{CL}(B,B')=\max_{(x,x')\in B\times B'}d(x,x')\leq d(x_0,y_0)$ and $\ell^{AL}(B,B')\leq d(x_0,y_0)$.

Therefore, $\ell^{CL}(B,B')=\ell^{AL}(B,B')=\ell^{SL}(B,B')$ and $\theta_{CL}(t_{i})=\theta_{AL}(t_{i})=\theta_{SL}(t_{i})$.  Since $\theta_{CL}(t_{m})=\theta_{AL}(t_{m})=\theta_{SL}(t_{m})=X$, this completes the proof.
\end{proof}

The following is part of Corollary 8.9 in \cite{M-P2}.

\begin{prop} $SL(\alpha)$ is faithful
\end{prop}

\begin{definicion} A  $HC$ method $\mathfrak{T}$ is \textbf{semi-stable in the Gromov-Hausdorff sense}  if for any sequence of finite metric spaces $((X_k,d_k))_{k\in \bn}$ in $(\mathcal{M},d_{GH})$ such that $\lim_{k\to \infty}(X_k,d_k)=(U,d)\in \mathcal{U}$ then $\lim_{k\to \infty}\mathfrak{T}_\mathcal{U}(X_k,d_k)=  \mathfrak{T}_{\mathcal{U}}(U,d)$.
\end{definicion}

Therefore, a faithful $HC$ method $\mathfrak{T}$ is semi-stable in the Gromov-Hausdorff sense if for any sequence of finite metric spaces $((X_k,d_k))_{k\in \bn}$ in $(\mathcal{M},d_{GH})$ such that $\lim_{k\to \infty}(X_k,d_k)=(U,d)\in \mathcal{U}$ then $\lim_{k\to \infty}\mathfrak{T}_\mathcal{U}(X_k,d_k)=  (U,d)$.

\begin{prop} $\mathfrak{T}^{SL}$ is semi-stable in the Gromov-Hausdorff sense.
\end{prop}

\begin{proof} This is an immediate consequence of  \cite[Proposition 26]{CM} (see \ref{Stable SL}).
\end{proof}

\begin{obs} There exists a sequence of metric spaces $((X_k,d_k))_{k\in \bn}$ in $(\mathcal{M},d_{GH})$ and a sequence of ultrametric spaces $(U_k,u_k)$ with 
$\lim_{k\to \infty}d_{GH}((X_k,d_k),(U_k,u_k))=0$ while $\lim_{k\to \infty}d_{GH}(\mathfrak{T}_{\mathcal{U}}^{CL}(X_k,d_k),(U_k,u_k))>0$. See example below.
\end{obs}

\begin{ejp} Let $X_k:=\{a_{i}\}_{-2\leq i \leq k}\cup \{b_{i}\}_{-2\leq i \leq k}$ and $d_k$ such that:

$d_k(a_{-2},b_{-2})=1$. 

$d_k(a_{-1},b_{-1})=1+\delta_k$ with $\delta_k =\frac{1}{k+1}$.

$d_k(a_{-2},a_{-1})=\frac{1}{2}$, $d_k(b_{-2},b_{-1})=\frac{1}{2}$.

$d_k(a_{0},a_{i})=1$, for $i=-2,-1$, $d_k(b_{0},b_{i})=1$, for $i=-2,-1$, $d_k(a_{0},b_{j})=1+\delta_k$ for $j<0$ and $d_k(a_{j},b_{0})=1+\delta_k$ for $j<0$.

Let $\varepsilon_k=\frac{\delta_k}{2}$. 

$d_k(a_{m},a_{i})=1+m\varepsilon_k$, for every $1\leq m\leq k$, $i<m$, $d_k(b_{m},b_{i})=1+m\varepsilon_k$, for every $m\geq 1$, $i<m$,  $d_k(a_{m},b_{i})=1+m\varepsilon_k+\delta_k $, for every $1\leq m\leq k$, $i\leq m$, and $d_k(b_{m},a_{i})=1+m\varepsilon_k + \delta_k$ for every $1\leq m\leq k$, $i\leq m$.

\begin{figure}[ht]
\centering
\includegraphics[scale=0.5]{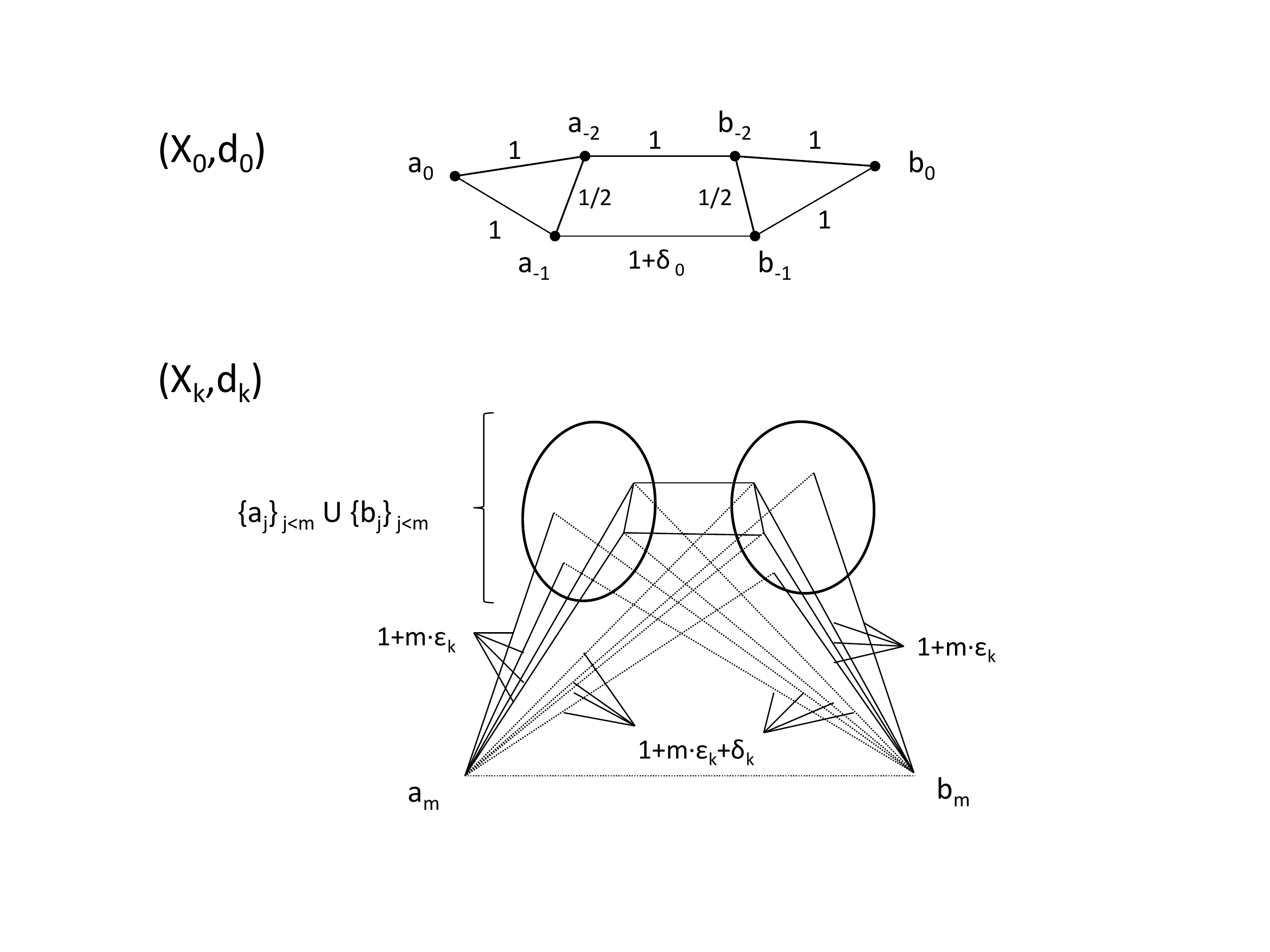}
\caption{$(X_0,d_0)$ is defined by the graph above. Below, the inductive construction of $(X_k,d_k)$ adding the points $a_m,b_m$ to the subset $\{a_{i}\}_{-2\leq i \leq m-1}\cup \{b_{i}\}_{-2\leq i \leq m-1}\subset (X_k,d_k)$ for every $1\leq m\leq k$.}
\label{Fig: Ejp_CL_1}
\end{figure}

Let $U_k:=\{\tilde{a}_{i}\}_{-2\leq i \leq k}\cup \{\tilde{b}_{i}\}_{-2\leq i \leq k}$ and $u_k\co U\times U\to \br$ as follows:

$u_k(\tilde{a}_{j},\tilde{b}_{j})=1$ for $j=-2,-1,0$.

$u_k(\tilde{a}_{-2},\tilde{a}_{-1})=\frac{1}{2}$, $u_k(\tilde{b}_{-2},\tilde{b}_{-1})=\frac{1}{2}$.

$u_k(\tilde{a}_{0},\tilde{a}_{i})=1$, for $i=-2,-1$, $u_k(\tilde{b}_{0},\tilde{b}_{i})=1$, for $i=-2,-1$

$u_k(\tilde{a}_{n},\tilde{a}_{i})=1+n\varepsilon_k$, for every $n\geq 1$, $i<n$, $u_k(\tilde{b}_{n},\tilde{b}_{i})=1+n\varepsilon_k$, for every $n\geq 1$, $i<n$,  $u_k(\tilde{a}_{n},\tilde{b}_{i})=1+n\varepsilon_k$, for every $n \geq 1$, $i\leq n$, and $u_k(\tilde{b}_{n},\tilde{a}_{i})=1+n\varepsilon_k $ for every $n\geq 1$, $i\leq n$.

\begin{figure}[ht]
\centering
\includegraphics[scale=0.5]{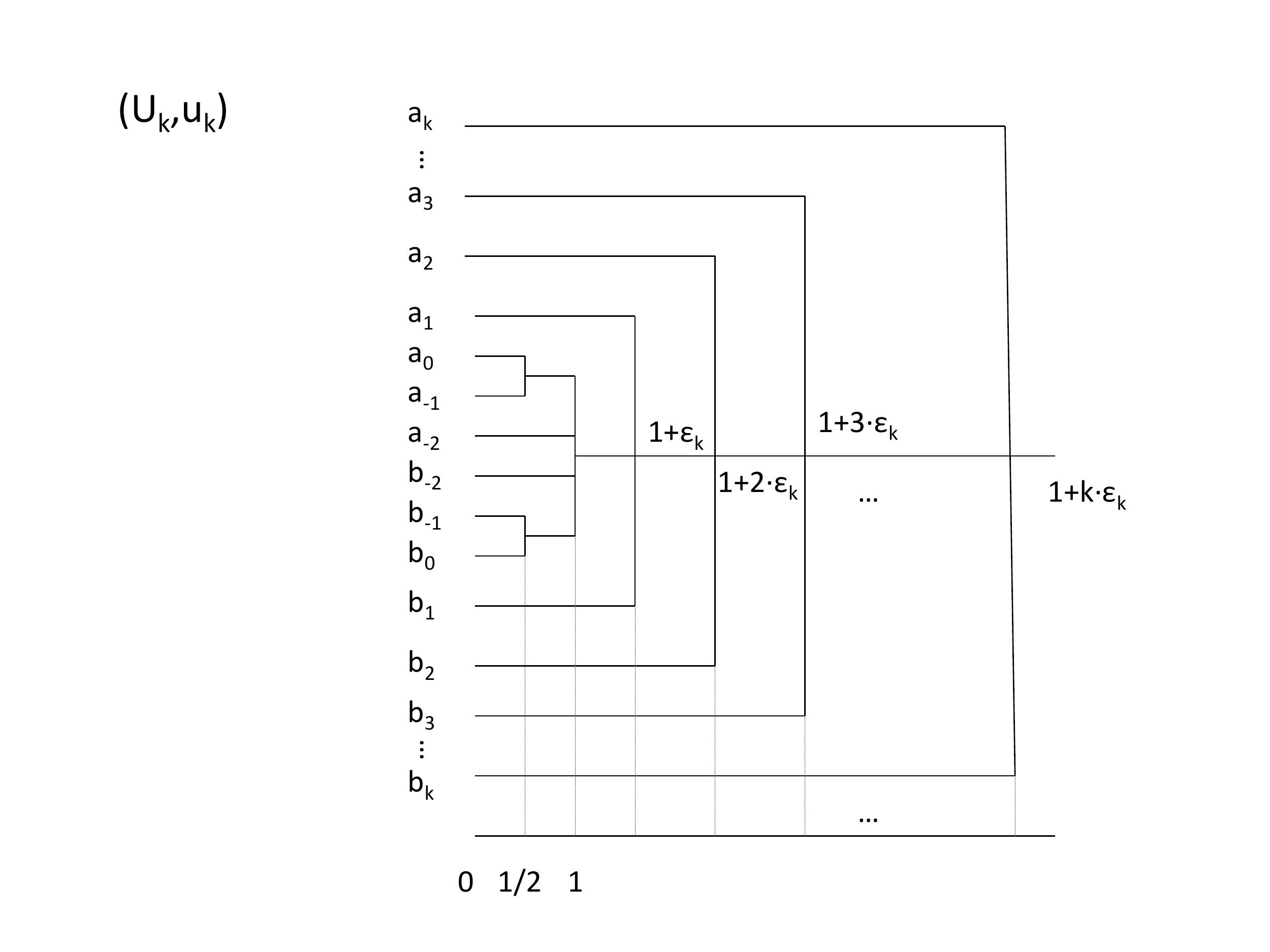}
\caption{The ultrametric space $(U_k,u_k)$ corresponds to the dendrogram represented in the figure.}
\label{Fig: Ejp_CL_2}
\end{figure}

Let us see that $d_{GH}((X_k,d_k),(U_k,u_k))\leq \delta_k=\frac{1}{k+1}$. It suffices to consider the correspondence $\tau=\{(a_i,\tilde{a}_i)\, | \, i\geq -2\}\cup\{(b_j,\tilde{b}_j)\, | \, j\geq -2\}$. Then, for every $(x,y),(x',y')\in \tau$, $|d_k(x,x')-u_k(y,y')|\leq \delta_k$.

Now, let us see that $u^k_{CL}(a_{-2},b_{-2})=1+(k)\varepsilon+\delta_k$ for every $k\geq 1$. For $t<\frac{1}{2}$, $\theta_{CL}(t)=\{a_{-2},...,a_k,b_{-2},...,b_k\}$. For $\frac{1}{2}\leq t<1$, it is trivial to check that $a_{-2}$ is merged with $a_{-1}$ and $b_{-2}$ is merged with $b_{-1}$. Therefore, $\theta_{CL}(t)=\{\{a_{-2},a_{-1}\},a_0,...,a_k,\{b_{-2},b_{-1}\},b_{0},...,b_k\}$. 

For $t=1$, notice that $\ell^{CL}(a_0,\{a_{-2},a_{-1}\})=1=\ell^{CL}(b_0,\{b_{-2},b_{-1}\})$ while $\ell^{CL}(\{a_{-2},a_{-1}\},\{b_{-2},b_{-1}\})=1+\delta_k>1$. The rest of maximal distances between blocks are also greater than 1. Thus, for every $1\leq t < 1+\varepsilon_k$, $\theta_{CL}(t)=\{\{a_{-2},a_{-1},a_0\},a_1,...,a_k,\{b_{-2},b_{-1},b_0\},b_{1},...,b_k\}$.

Suppose that for some $1\leq m <k$ and for every $1+(m-1)\varepsilon_k \leq t <1+m\varepsilon_k$ we have that $\theta_{CL}(t)=\{\{a_{-2},...,a_{m-1}\},a_m,...,a_k,\{b_{-2},...,b_{m-1}\},b_{m},...,b_k\}$. Again, $\ell^{CL}(a_m,\{a_{-2},...,a_{m-1}\}\})=1+m\varepsilon_k$, $\ell^{CL}(b_m,\{b_{-2},...,b_{m-1}\})=1+m\varepsilon_k$ and $\ell^{CL}(\{a_{-2},...,a_{m-1}\},\{b_{-2},...,b_{m-1}\})=1+m\varepsilon_k+\delta_k$. The rest of maximal distances between blocks are also greater than $1+m\varepsilon_k$. Therefore, for every $1+m\varepsilon_k \leq t <1+(m+1)\varepsilon_k$ we have that  $\theta_{CL}(t)=\{\{a_{-2},...,a_{m}\},a_{m+1},...,a_k,\{b_{-2},...,b_{m}\},b_{m+1},...,b_k\}$.

Thus, $u^k_{CL}(a_{-2},b_{-2})=1+k\varepsilon+\delta_k=1+k\frac{1}{2(k+1)}+\frac{1}{k+1}>\frac{3}{2}$ while $u_k(\tilde{a}_{-2},\tilde{b}_{-2})=1$.
\end{ejp}

\textbf{Notation}: Given two metrics $d,d'$ defined on a set $X$, let us denote $d\leq d'$ if $d(x,x')\leq d'(x,x')$ $\forall \, x,x'\in X$.

Given a $HC$ method $\mathfrak{T}=(\ell,\emptyset)$, let us recall the notations: $\mathfrak{T}_{\mathcal{D}}(X,d)=\theta_X$ and $\mathfrak{T}_{\mathcal{U}}(X,d)=(X,u_X)$.

\begin{definicion} A $HC$ method $\mathfrak{T}=\mathfrak{T}(\ell,\emptyset)$ is \textbf{normal} if for every metric space, $X$, and  $\forall \, B_1,B_2\in \theta_X(t)$, then $\ell(B_1,B_2)\leq \ell^{CL}(B_1,B_2)$ and $u_X\geq u_{SL}$. 
\end{definicion}

\begin{lema}\label{Lemma: t-connected} $u_X\geq u_{SL}$ if and only if $\forall$ $B\in \theta_X(t)$, $B$ is $t$-connected.
\end{lema}

\begin{proof} If $u_X\geq u_{SL}$, it is immediate to check that $\forall \, x,y\in  X$, if $u_X(x,y)\leq t$, $u_{SL}(x,y)\leq t$ and $x,y$ are in the same t-component.

Let $x,y\in X$ such that $u_X(x,y)=t$. Then, there exist some $B\in \theta_X(t)$ such that $x,y\in B$. If  $\forall$ $B\in \theta_X(t)$, $B$ is $t$-connected, then $B$ is contained in some block of $\theta_{SL}(t)$. Therefore, $u_{SL}(x,y)\leq t$.  
\end{proof}

\begin{lema}\label{Lemma: max} Let $\mathfrak{T}=\mathfrak{T}(\ell,\emptyset)$ be such that for every metric space, $X$, and any $t>0$, $\forall \, B_1,B_2\in \theta_X(t)$, $\ell(B_1,B_2)\leq \ell^{CL}(B_1,B_2)$. Then, $\forall \, t>0$ and every $B_1,B_2\in \theta_X(t)$, $\max\{d(x,y)\, | \, x\in B_1, \ y\in B_2\}>t$.
\end{lema}

\begin{proof} Let $k_i:=max\{k_r\leq t : k_r\in \Delta(X,u_X)\} $. Then, $B_1,B_2$ are blocks in $\theta_X(t)=\theta_X(k_i)$. 

If $\max\{d(x,y)\, | \, x\in B_1, \ y\in B_2\}\leq t$, then $\ell(B_1,B_2)\leq \ell^{CL}(B_1,B_2)\leq t$. Therefore, $B_1\sim_{\ell,t}B_2$ which contradicts the fact that $B_1,B_2$ are blocks in $\theta_X(t)$.

Hence, $\max\{d(x,y)\, | \, x\in B_1, \ y\in B_2\}>t$. 
\end{proof}

\begin{lema}\label{Lemma: normal} Let $(X,d)$ be a finite metric space with $\Delta(X,d)=\{t_i \, : \, 0\leq i\leq n\}$ and $(U,u)$ be an ultrametric space with $d_{GH}((X,d),(U,u))<\frac{\delta}{2}$. Let $\mathfrak{T}=\mathfrak{T}(\ell,\emptyset)$ be a normal $HC$ method. Given $x,x'\in X$ with $d(x,x')=t_i$, if $u_{X}(x,x')> t_{j-1}\geq t_i$, then $t_j\leq t_{j-1}+2\delta$.
\end{lema}

\begin{proof} Let $B_1,B_2\in \theta_{X}(t_{j-1})$ with $x\in B_1$ and $x'\in B_2$. Let $k_{r}\in \Delta(\mathfrak{T}_{\mathcal{U}}(X,d))$ be such that $k_{r-1}\leq t_{j-1}<k_r$. Then, $\theta_{X}(t_{j-1})=\theta_{X}(k_{r-1})$. 

Let $a\in B_1$ and $a'\in B_2$ with $d(a,a')=\max\{d(x,y) \, | \, x\in B_1, \ y\in B_2\}$. Since $B_1,B_2$ are different blocks in $\theta_{X}(t_{j-1})$, by Lemma \ref{Lemma: max}, $d(a,a')>t_{j-1}\geq k_{r-1}$ and $d(a,a')\geq t_j$. 

Let $\tau$ be a correspondence such that $\sup_{(z,u)(z',u')\in \tau}\Gamma_{X,U}(z,u,z',u')<\delta$ and let $(a,b),(a',b'),(x,y),(y,y')\in \tau$. See Figure \ref{Fig: Lema_CL}. By Lemma \ref{Lemma: t-connected}, every block in $\theta_{X}(t_{j-1})$ is $t_{j-1}$-connected and there exist $t_{j-1}$-chains in $(X,d)$ from $a$ to $x$ and from $a'$ to $x'$. Since $d(x,x')=t_i\leq t_{j-1}$, there is a $t_{j-1}$-chain from $a$ to $a'$. 

Therefore, since $\sup_{(z,u)(z',u')\in \tau}\Gamma_{X,U}(z,u,z',u')<\delta$, there is a $(t_{j-1}+\delta)$-chain from $b$ to $b'$. By the properties of the ultrametric, this implies that $u(b,b')\leq t_{j-1}+\delta$ and, therefore, $d(a,a')\leq t_{j-1}+2\delta$.

Since $u_{X}(x,x')\geq k_r$, then $t_{j-1}<k_r\leq \ell(B_1,B_2)\leq d(a,a')<t_{j-1}+2\delta$. In particular, $t_{j}<t_{j-1}+2\delta$.
\end{proof}

\begin{figure}[ht]
\centering
\includegraphics[scale=0.5]{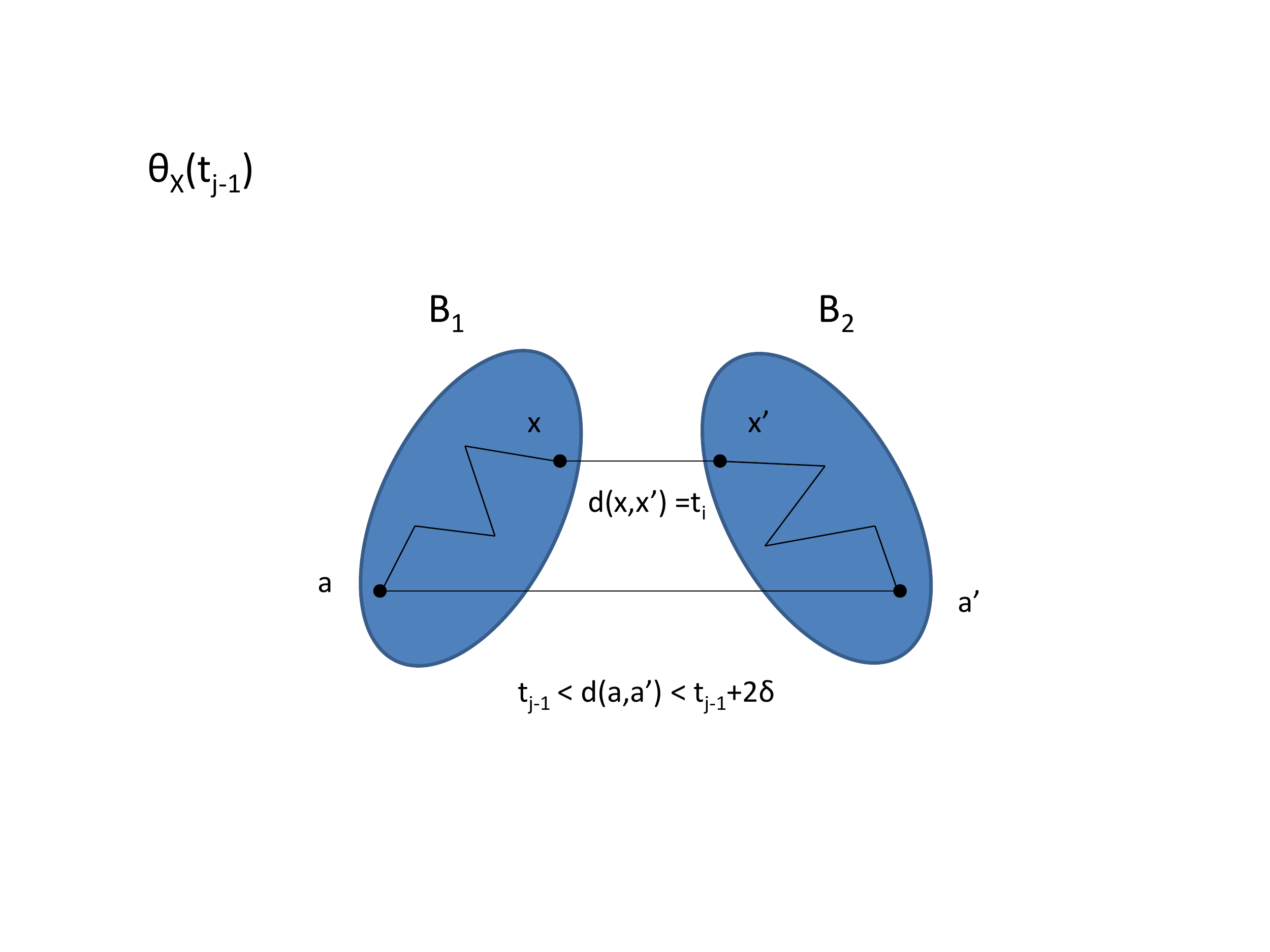}
\caption{$B_1,B_2\in \theta_{X}(t_{j-1})$ with $t_{j-1}\geq t_i=d(x,x')$ and $d(a,a')\geq \ell(B_1,B_2)$.}
\label{Fig: Lema_CL}
\end{figure}

The following proposition is a particular case of Theorem \ref{Th: semistable}. We include this proposition and the proof of this particular case because we feel that it may help to clarify the proof of the theorem.

\begin{prop}\label{Prop: semi-stable} Let $(U,u)$ be an ultrametric space such that $\Delta(U,u)$ is finite (in particular, if $(U,u)$ is a finite ultrametric space). Let  $\mathfrak{T}=(\ell,\emptyset)$ be any normal standard linkage-based $HC$ method. Then, for any sequence of finite metric spaces $((X_k,d_k))_{k\in \bn}$ in $(\mathcal{M},d_{GH})$ such that $\lim_{k\to \infty}(X_k,d_k)=(U,d)\in \mathcal{U}$, $\lim_{k\to \infty}\mathfrak{T}_\mathcal{U}(X_k,d_k)=  (U,d)$.
\end{prop}

\begin{proof}  Let $\Delta(U,u)=\{u_0,...,u_p\}$ and let $0<\varepsilon<u_i-u_{i-1}$ for every $1\leq i \leq p$. Let $\delta<\frac{1}{4}\varepsilon$ and suppose $d_{GH}((X_k,d_k),(U,u))<\frac{\delta}{2}$. Then, there is a correspondence $\tau$ such that for any $(x,y),(x',y')\in \tau$, $|d_k(x,x')-u(y,y')|<\delta$.

Let $\mathfrak{T}_{\mathcal{U}}(X_k,d_k)=(X_k,u_k)$ and $\mathfrak{T}_{\mathcal{D}}(X_k,d_k)=\theta_{k}$.

Claim: $u(y,y')\leq u_{k}(x,x')+\delta < u_{k}(x,x')+\varepsilon$. 

Let $u_{k}(x,x')=k_r\in \Delta(X_k,u_{k})$. Then, there exist $B_0,B_1,...,B_n\in \theta_{k}(k_{r-1})$ with $x\in B_0$, $x'\in B_n$ and $\ell(B_{j-1},B_j)\leq k_r$ for every $1\leq j \leq n$. By Lemma \ref{Lemma: t-connected} each $B_j$ is $(k_{r-1})$-connected. Then, there is a $k_r$-chain in $(X_k,d_k)$ from $x$ to $x'$: $x=x_0,...,x_n=x'$. Let $(x_j,y_j)\in R$ for every $0\leq j \leq n$. Then, since $|d_k(x_{j-1},x_j)-u(y_{j-1},y_j)|<\delta$, $y=y_0,...,y_n=y'$ is a $(k_r+\delta)$-chain in $(U,u)$. Since $(U,u)$ is ultrametric, this implies that $u(y,y')\leq k_r+\delta$.

Claim: $u_{k}(x,x')\leq u(y,y') +\varepsilon$. 

Let $(x,y),(x',y')\in R$. Let $d_k(x,x')=t_i$ and $u(y,y')=t$. We already know that $|t_i-t|<\delta$.

If $u_{k}(x,x')\leq t_i$, then $u_{k}(x,x')\leq d_k(x,x')\leq u(y,y')+\delta$ and we are done. Suppose $u_{k}(x,x')=k_r\in [t_{i+m-1},t_{i+m})$ with $m\geq 1$. By Lemma \ref{Lemma: normal}, for every $1\leq k \leq m$, $t_{i+k}-t_{i+k-1}\leq 2\delta$.

For every $i\leq j < i+m$, consider $B_1^{t_j},B_2^{t_j}\in \theta_{k}(t_j)$ with $x\in B_1^{t_j}$, $x'\in B_2^{t_j}$ and let $a_j\in B_1^{t_{j}}$ and $a'_j\in B_2^{t_{j}}$ with $d_k(a_j,a'_j)=\max\{d_k(x,y) \, | \, x\in B_1^{t_{j}}, \ y\in B_2^{t_{j}}\}$. The situation is the same from Lemma \ref{Lemma: normal}, see Figure \ref{Fig: Lema_CL}. Let $(a_j,b_j),(a'_j,b'_j)\in R$. Hence, for every $i\leq j < i+m$, $d_k(a_{j+1},a'_{j+1})-d_k(a_j,a'_j)\leq t_{j}+2\delta-t_{j}=2\delta$. Therefore, $u(b_{j+1},b'_{j+1})-u(b_j,b'_j)\leq 4\delta$. Since $4\delta<\varepsilon$, it follows that $u(b_{j+1},b'_{j+1})=u(b_j,b'_j)$ for every $i\leq j < i+m$. Then, $d(a_{i+m},a'_{i+m})-d(a_i,a'_i)\leq  2\delta$ and $u_{k}(x,x')\leq t_i+2\delta \leq u(y,y')+3\delta <u(y,y')+ \varepsilon$ proving the claim.
\end{proof}

In the previous proposition it is possible to eliminate the condition on the ultrametric space. The main idea of the proof is the same in spite of some technical difficulties.

\begin{lema}\label{Lema: Ultram} Let $(X,d)$ be a finite metric space, $(U,u)$ be an ultrametric space and let $\varepsilon>0$. If $d_{GH}((X,d),(U,u))<\frac{\varepsilon}{2}$, then $\Delta(U,u)$ has a finite number of distances greater than $\varepsilon$.
\end{lema}

\begin{proof} Suppose $d_{GH}((X_k,d_k),(U,u))<\frac{\varepsilon}{2}$. Then, there is a correspondence $\tau$ such that for any $(x,y),(x',y')\in \tau$, $|d_k(x,x')-u(y,y')|<\varepsilon$. Since $X_k$ is finite, let $X_k=\{x_1,...,x_n\}$ and let $(x_i,y_i)\in \tau$. Then, $\{B(y_i,\varepsilon)\, | \, 1\leq i\leq n\}$ defines a finite covering of $U$. Therefore, by the properties of the ultrametric, $\Delta(U,u)$ has a finite number of distances greater than $\gamma$. 
\end{proof}

\begin{teorema}\label{Th: semistable} Every normal, faithful, standard linkage-based $HC$ method, $\mathfrak{T}=(\ell,\emptyset)$, is semi-stable in the Gromov-Hausdorff sense.
\end{teorema}

\begin{proof} By Lemma \ref{Lema: Ultram}, since $\lim_{k\to \infty}(X_k,d_k)=(U,u)$ with $X_k$ finite,
$\Delta_{\geq \varepsilon}(U,u)=\{u(x,y) \, | \, x,y \in U \mbox{ with } u(x,u)\geq \varepsilon\}$ is a finite set for every $\varepsilon >0$.

Let $\varepsilon >0$. Let $\Delta_{\geq \frac{\varepsilon}{2}}(U,u)=\{l_0,l_1,...,l_p\}$ with $l_i<l_j$ for every $i<j$. 
Let $0<\delta <\frac{1}{4}\min_{1\leq j \leq m}\{l_j-l_{j-1}\}$ and suppose also that  $0<\delta<\frac{\varepsilon}{12}$.

Let $(X_k,d_k)$ be such that $d_{GH}((X_k,d_k),(U,u))<\frac{\delta}{2}$. Then, since $\mathfrak{T}$ is faithful, it suffices to check that $d_{GH}(\mathfrak{T}_\mathcal{U}(X_k,d_k),(U,u))<\varepsilon$.

Let $\tau$ be a correspondence such that $\forall \, (x,y), \, (x',y')\in \tau$, $|d(x,x')-u(y,y')|<\delta$.

Let $\mathfrak{T}_{\mathcal{U}}(X_k,d_k)=(X_k,u_k)$ and $\mathfrak{T}_{\mathcal{D}}(X_k,d_k)=\theta_{k}$.

Claim: $u(y,y')\leq u_{k}(x,x')+\delta < u_{k}(x,x')+\varepsilon$. This claim is proved by the same argument used 
in Proposition \ref{Prop: semi-stable}.

Let $u_{k}(x,x')=k_r\in \Delta(X_k,u_{k})$. Then, there exist $B_0,B_1,...,B_n\in \theta_{k}(k_{r-1})$ with $x\in B_0$, $x'\in B_n$ and $\ell(B_{j-1},B_j)\leq k_r$. By Lemma \ref{Lemma: t-connected}  $B_j$ is $(k_{r-1})$-connected. Then, there is a $k_r$-chain in $(X_k,d_k)$ from $x$ to $x'$: $x=x_0,...,x_n=x'$. Let $(x_j,y_j)\in \tau$ for every $0\leq j \leq n$. Then, since $|d_k(x_{j-1},x_j)-u(y_{j-1},y_j)|<\delta$, $y=y_0,...,y_n=y'$ is a $(k_r+\delta)$-chain in $(U,u)$. Since $(U,u)$ is ultrametric, this implies that $u(y,y')\leq k_r+\delta$.

Claim: $u_{k}(x,x')\leq u(y,y') +\varepsilon$. 

Let $(x,y),(x',y')\in \tau$. Let $d(x,x')=t_i$ and $u(y,y')=t$. Then, $|t_i-t|<\delta$.

If $u_{k}(x,x')\leq t_i$ we are done since $u_{k}(x,x')\leq d(x,x')\leq u(y,y')+\delta$. Then, suppose $u_{k}(x,x')=k_r\in [t_{i+m-1},t_{i+m})$ with $m\geq 1$. By Lemma \ref{Lemma: normal}, for every $1\leq k \leq m$, $t_{i+k}-t_{i+k-1}\leq 2\delta$.

If $t_i\geq \frac{\varepsilon}{2}+\delta$, consider for every $i\leq j < i+m$, as in the proof of Lemma \ref{Lemma: normal}, $B_1^{t_j},B_2^{t_j}\in \theta_{k}(t_j)$ with $x\in B_1^{t_j}$, $x'\in B_2^{t_j}$ and let $a_j\in B_1^{t_{j}}$ and $a'_j\in B_2^{t_{j}}$ with $d(a_j,a'_j)=\max\{d(x,y) \, | \, x\in B_1^{t_{j}}, \ y\in B_2^{t_{j}}\}$. Let $(a_j,b_j),(a'_j,b'_j)\in \tau$. Hence, for every $i\leq j < i+m$, $d(a_{j+1},a'_{j+1})-d(a_j,a'_j)\leq t_{j}+2\delta-t_{j}=2\delta$. Therefore, $u(b_{j+1},b'_{j+1})-u(b_j,b'_j)\leq 4\delta$. Since  $u(b_j,b'_j)>t_j-\delta >\frac{\varepsilon}{2}$ and $4\delta <\min_{1\leq j \leq m}\{l_j-l_{j-1}\}$, then $u(b_{j+1},b'_{j+1})=u(b_j,b'_j)$ for every $j$. Thus, $d(a_{i+m},a'_{i+m})-d(a_i,a'_i)\leq u(b_{i+m},b'_{i+m})-u(b_i,b'_i) +2\delta= 2\delta$ and $u_{X_k}(x,x')\leq t_i+2\delta \leq u(y,y')+3\delta$ proving the claim.

Finally, if $t_i< \frac{\varepsilon}{2}+\delta$, then, either $u_{X_k}(x,x')\leq \varepsilon$ and we are done or, by Lemma \ref{Lemma: normal}, there is some $t_i<t_k\in (\frac{\varepsilon}{2}+\delta,\frac{\varepsilon}{2}+3\delta)$. Consider, as above, for every $k\leq j < i+m$ the pairs $(a_j,a'_j)$ with $d(a_{j+1},a'_{j+1})-d(a_j,a'_j)\leq 2\delta$, $(a_j,b_j),(a'_j,b'_j)\in \tau$. Hence, as above, $u(b_{k+m},b'_{k+m})=u(b_k,b'_k)$ for every $j$ and $d(a_{k+m},a'_{k+m})-d(a_k,a'_k)\leq  2\delta$. Therefore, $u_{X_k}(x,x')\leq t_k+2\delta \leq \frac{\varepsilon}{2}+5\delta<\varepsilon$ proving the claim.
\end{proof}

\begin{cor} $\mathfrak{T}^{CL}$ is semi-stable in the Gromov-Hausdorff sense.
\end{cor}

\begin{cor} $\mathfrak{T}^{AL}$ is semi-stable in the Gromov-Hausdorff sense.
\end{cor}

The proof of Theorem \ref{Th: semistable} can be adapted to prove that $SL(\alpha)$ is also semi-stable. First, let us see that the analogue to Lemma \ref{Lemma: normal} also works for $SL(\alpha)$.

\begin{lema}\label{Lema: normal-alpha}  Let $(X,d)$ be a finite metric space with $\Delta(X,d)=\{t_0<\cdots <t_n\}$ and $(U,u)$ be an ultrametric space with $d_{GH}((X,d),(U,u))<\frac{\delta}{2}$. Let $\mathfrak{T}^{SL(\alpha)}_\mathcal{D}(X,d)=\theta_\alpha$ and $\mathfrak{T}^{SL(\alpha)}_\mathcal{U}(X,d)=u_\alpha$ for any $\alpha \geq 1$. Given $x,x'\in X$ with $d(x,x')=t_i$, if $u_{\alpha}(x,x')> t_{j-1}\geq t_i$, then $t_j\leq t_{j-1}+2\delta$.
\end{lema}

\begin{proof} Let $B_1,B_2\in \theta_{\alpha}(t_{j-1})$ with $x\in B_1$ and $x'\in B_2$.

Let $a\in B_1$ and $a'\in B_2$ with $d(a,a')=\max\{d(x,y) \, | \, x\in B_1, \ y\in B_2\}$. Then, it follows from the construction of $\theta_{\alpha}(t_{j-1})$ that $d(a,a')\geq t_j >t_{j-1}$.  Suppose that $d(a,a')\leq t_{j-1}$. Then, let $B_a,B_{a'}$ two blocks in $\theta_\alpha(t_{j-2})$ such that $a\in B_a$, $a'\in B_{a'}$. Clearly, $\ell^{SL}(B_a,B_{a'})\leq t_{j-1}$ and, since $B_a\subset B_1$ and $B_{a'}\subset B_2$, then $d(a,a')=\max\{d(x,y) \, | \, x\in B_1, \ y\in B_2\}\geq \max\{d(x,y) \, | \, x\in B_a, \ y\in B_{a'}\}$. Therefore, $F_{t_{j-1}}(B_a\cup B_{a'})$ has dimension $\sharp (B_a \cup B_{a'})-1\geq \min\{dim(F_{t_{j-1}}(B_a)),dim(F_{t_{j-1}}(B_{a'}))\}$. Therefore, $B_a$ and $B_{a'}$ are merged in $\theta_\alpha(t_{j-1})$ for any $\alpha \geq 1$ and $a,a'$ are contained in the same block in $\theta_{\alpha}(t_{j-1})$ which is a contradiction.

Let $\tau$ be a correspondence such that $\sup_{(z,u)(z',u')\in \tau}\Gamma_{X,U}(z,u,z',u')<\delta$ and let $(a,b),(a',b'),(x,y),(y,y')\in \tau$. See Figure \ref{Fig: Lema_CL}. It is immediate to see that every block in $\theta_{\alpha}(t_{j-1})$ is $t_{j-1}$-connected and there exist $t_{j-1}$-chains in $(X,d)$ from $a$ to $x$ and from $a'$ to $x'$. Since $d(x,x')=t_i\leq t_{j-1}$, there is a $t_{j-1}$-chain from $a$ to $a'$. 

Therefore, there is a $(t_{j-1}+\delta)$-chain from $b$ to $b'$. By the properties of the ultrametric, this implies that $u(b,b')\leq t_{j-1}+\delta$ and, hence, $d(a,a')\leq t_{j-1}+2\delta$. In particular, $t_{j-1}<t_j\leq d(a,a')\leq t_{j-1}+2\delta$. 
\end{proof}

\begin{prop} $\mathfrak{T}^{SL(\alpha)}$ is semi-stable in the Gromov-Hausdorff sense.
\end{prop}

\begin{proof} By Lemma \ref{Lema: Ultram}, since $\lim_{k\to \infty}(X_k,d_k)=(U,u)$ with $X_k$ finite,
$\Delta_{\geq \varepsilon}(U,u)=\{u(x,y) \, | \, x,y \in U \mbox{ with } u(x,u)\geq \varepsilon\}$ is a finite set for every $\varepsilon >0$.

Let $\varepsilon >0$. Let $\Delta_{\geq \frac{\varepsilon}{2}}(U,u)=\{l_0,l_1,...,l_p\}$ with $l_i<l_j$ for every $i<j$.

Let $0<\delta <\frac{1}{4}\min_{1\leq j \leq m}\{l_j-l_{j-1}\}$ and suppose also that  $0<\delta<\frac{\varepsilon}{12}$.

Let $(X_k,d_k)$ be such that $d_{GH}((X_k,d_k),(U,u))<\frac{\delta}{2}$. Then, since $SL(\alpha)$ is faithful, it suffices to check that $d_{GH}(\mathfrak{T}_\mathcal{U}(X_k,d_k),(U,u))<\varepsilon$.

Let $\tau$ be a correspondence such that $\forall \, (x,y), \, (x',y')\in \tau$, $|d(x,x')-u(y,y')|<\delta$.

Let $\mathfrak{T}^{SL(\alpha)}_{\mathcal{U}}(X_k,d_k)=(X_k,u_k)$ and $\mathfrak{T}^{SL(\alpha)}_{\mathcal{D}}(X_k,d_k)=\theta_{k}$.

Claim: $u(y,y')\leq u_{k}(x,x')+\delta < u_{k}(x,x')+\varepsilon$.

Let $u_{k}(x,x')=k_r\in \Delta(X_k,u_{k})$. Then, there exist $B_0,B_1,...,B_n\in \theta_{k}(k_{r-1})$ with $x\in B_0$, $x'\in B_n$ and $\ell^{SL}(B_{j-1},B_j)\leq k_r$. By Lemma \ref{Lemma: t-connected}  $B_j$ is $(k_{r-1})$-connected. Then, there is a $k_r$-chain in $(X_k,d_k)$ from $x$ to $x'$: $x=x_0,...,x_n=x'$. Let $(x_j,y_j)\in \tau$ for every $0\leq j \leq n$. Then, since $|d_k(x_{j-1},x_j)-u(y_{j-1},y_j)|<\delta$, $y=y_0,...,y_n=y'$ is a $(k_r+\delta)$-chain in $(U,u)$. Since $(U,u)$ is ultrametric, this implies that $u(y,y')\leq k_r+\delta$.

Claim: $u_{k}(x,x')\leq u(y,y') +\varepsilon$. 

Let $(x,y),(x',y')\in \tau$. Let $d(x,x')=t_i$ and $u(y,y')=t$. Then, $|t_i-t|<\delta$.

If $u_{k}(x,x')\leq t_i$ we are done since $u_{k}(x,x')\leq d(x,x')\leq u(y,y')+\delta$. Then, suppose $u_{k}(x,x')=k_r\in [t_{i+m-1},t_{i+m})$ with $m\geq 1$. By Lemma \ref{Lema: normal-alpha}, for every $1\leq k \leq m$, $t_{i+k}-t_{i+k-1}\leq 2\delta$.

If $t_i\geq \frac{\varepsilon}{2}+\delta$, consider for every $i\leq j < i+m$,  $B_1^{t_j},B_2^{t_j}\in \theta_{k}(t_j)$ with $x\in B_1^{t_j}$, $x'\in B_2^{t_j}$ and let $a_j\in B_1^{t_{j}}$ and $a'_j\in B_2^{t_{j}}$ with $d(a_j,a'_j)=\max\{d(x,y) \, | \, x\in B_1^{t_{j}}, \ y\in B_2^{t_{j}}\}$. Let $(a_j,b_j),(a'_j,b'_j)\in \tau$. Hence, for every $i\leq j < i+m$, $d(a_{j+1},a'_{j+1})-d(a_j,a'_j)\leq t_{j}+2\delta-t_{j}=2\delta$. Therefore, $u(b_{j+1},b'_{j+1})-u(b_j,b'_j)\leq 4\delta$. Since  $u(b_j,b'_j)>t_j-\delta >\frac{\varepsilon}{2}$ and $4\delta <\min_{1\leq j \leq m}\{l_j-l_{j-1}\}$, then $u(b_{j+1},b'_{j+1})=u(b_j,b'_j)$ for every $j$. Thus, $d(a_{i+m},a'_{i+m})-d(a_i,a'_i)\leq u(b_{i+m},b'_{i+m})-u(b_i,b'_i) +2\delta= 2\delta$ and $u_{X_k}(x,x')\leq t_i+2\delta \leq u(y,y')+3\delta$ proving the claim.

Finally, if $t_i< \frac{\varepsilon}{2}+\delta$, then, either $u_{X_k}(x,x')\leq \varepsilon$ and we are done or, by Lemma \ref{Lema: normal-alpha}, there is some $t_i<t_k\in (\frac{\varepsilon}{2}+\delta,\frac{\varepsilon}{2}+3\delta)$. Consider, as above, for every $k\leq j < i+m$ the pairs $(a_j,a'_j)$ with $d(a_{j+1},a'_{j+1})-d(a_j,a'_j)\leq 2\delta$, $(a_j,b_j),(a'_j,b'_j)\in R$. Hence, as above, $u(b_{k+m},b'_{k+m})=u(b_k,b'_k)$ for every $j$ and $d(a_{k+m},a'_{k+m})-d(a_k,a'_k)\leq  2\delta$. Therefore, $u_{X_k}(x,x')\leq t_k+2\delta \leq \frac{\varepsilon}{2}+5\delta<\varepsilon$ proving the claim.
\end{proof}

\section{Stability of HC methods.}\label{Section: stable}

\begin{definicion} A $HC$ method $\mathfrak{T}$ is \textbf{stable in the Gromov-Hausdorff sense}  if 
\[\mathfrak{T}_\mathcal{U}\co (\mathcal{M},d_{GH})\to (\mathcal{U},d_{GH})\] is continuous.
\end{definicion}

Thus, $SL$ $HC$ is stable in the Gromov-Hasudorff sense.

\textbf{Notation:} Let $\mathcal{H}(\mathfrak{T},\varepsilon)$ be the set of finite metric spaces, $(X,d)$, such that  $\theta_X(\varepsilon)$ has exactly two blocks, where $\theta_X=\mathfrak{T}_\mathcal{D}(X,t)$. In particular, $\mathcal{H}(\mathfrak{T}^{SL},\varepsilon)$ denotes the set of finite metric spaces with exactly two $\varepsilon$-components.

\textbf{Notation:} Let $I(a_1,...,a_n)$ denote the isometry type of the finite metric space defined by the points $\{\bar{p}_0,\bar{p}_1,...,\bar{p}_n\}$ where  $d(\bar{p}_i,\bar{p}_j)=a_{i+1}+\cdots+a_j$ for every $i<j$ (this is, the isometry type of the set $\{0,a_1,a_1+a_2,...,a_1+a_2+\cdots+a_n\}\subset \br$ with the euclidean metric). See Figure \ref{Fig: Interval}. In particular, $$I(R)=\Big(\{\bar{p}_0,\bar{p}_1\}, \left( 
\begin{array}{ccc}
R &  0 \\ 
0 & R \end{array}
\right)\Big)$$

\begin{figure}[ht]
\centering
\includegraphics[scale=0.4]{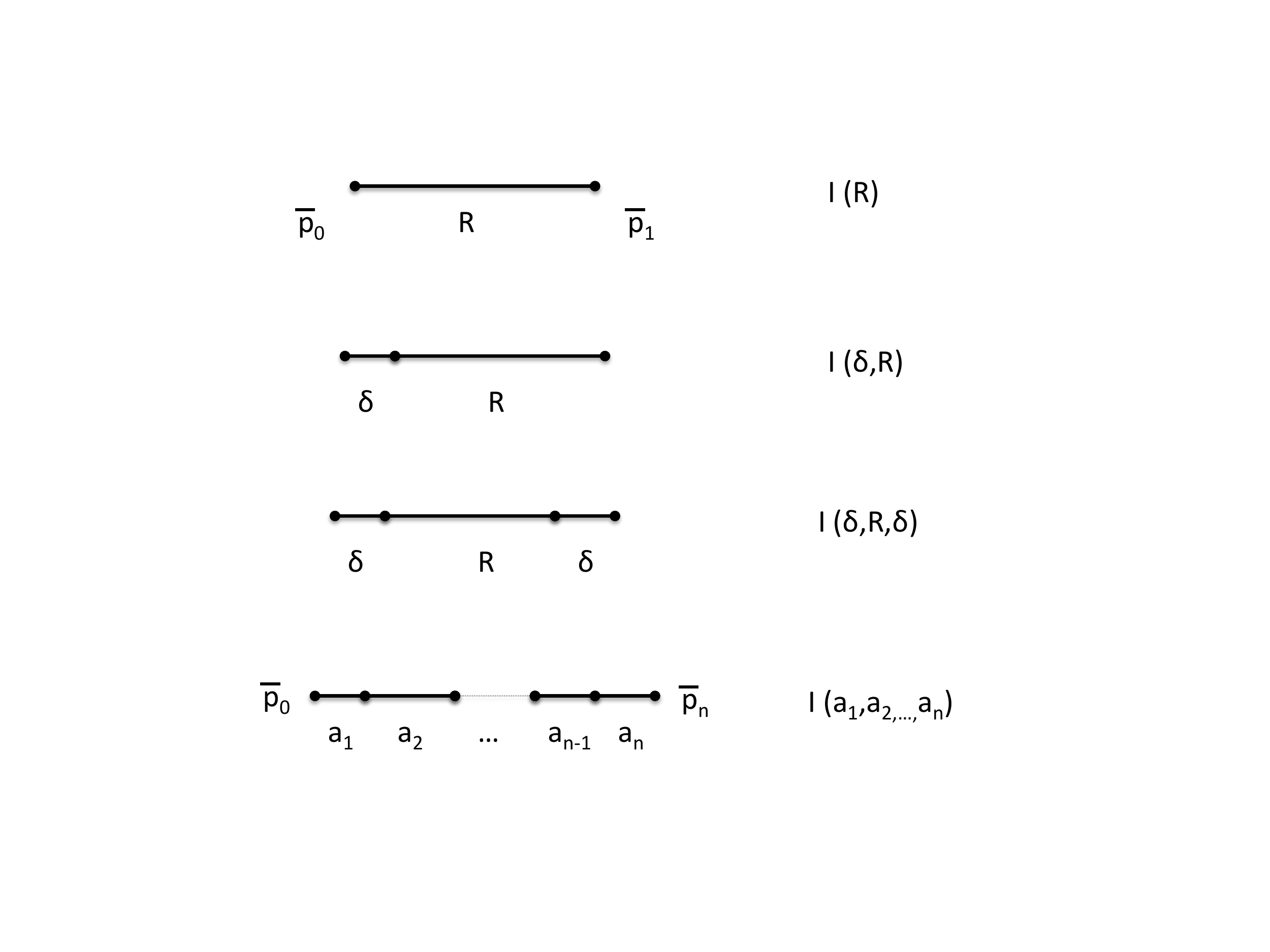}
\caption{$I(a_1,...,a_n)$ denotes the metric space defined by the vertices of the graph above. The edges have length $a_1,...,a_n$ and the distance between two vertices is the length of the minimal path connecting them.}
\label{Fig: Interval}
\end{figure}

\begin{ejp} $\mathfrak{T}^{AL}$ and $\mathfrak{T}^{CL}$ are not stable in the Gromov-Hausdorff sense.

Consider $(X,d)=I(1,1)$ and $(Y,d')=I(1,1+\delta)$ with $\delta>0$ arbitrarily small. 

First, notice that $d_{GH}((X,d),(Y,d'))=\delta$.

It is immediate to check that $\mathfrak{T}^{AL}_{\mathcal{U}}(X)=\mathfrak{T}^{CL}_{\mathcal{U}}(X)=(X,u)$ with  $X=\{x_1,x_2,x_3\}$ and $u(x_i,x_j)=1$ for every $i\neq j$. 

Let $\mathfrak{T}^{AL}_{\mathcal{U}}(Y)=(Y,u_1)$ and $\mathfrak{T}^{CL}_{\mathcal{U}}(Y)=(Y,u_2)$. Now, $Y=\{y_1,y_2,y_3\}$, $u_1(y_1,y_2)=1$ and $u_1(y_1,y_3)=u_1(y_2,y_3)=\frac{3+2\delta}{2}$. Also, $u_2(y_1,y_2)=1$ and $u_2(y_1,y_3)=u_2(y_2,y_3)=2+\delta$.

Thus, it can be seen that $d_{GH}(\mathfrak{T}^{AL}_{\mathcal{U}}(X,d),\mathfrak{T}^{AL}_{\mathcal{U}}(Y,d'))=\frac{1+2\delta}{4}>\frac{1}{4}$. Similarly, $d_{GH}(\mathfrak{T}^{CL}_{\mathcal{U}}(X,d),\mathfrak{T}^{CL}_{\mathcal{U}}(Y,d'))=\frac{2+\delta}{2}>1$. Therefore, $AL$ and $CL$ are not stable.  
\end{ejp}

\begin{definicion} Let $\mathfrak{T}=\mathfrak{T}(\ell;P)$ be a $HC$ method. We say that \textbf{$P$ is nontrivial for $\ell$} if there exists some metric space $(X,d)$ such that for $\mathfrak{T}_D(X,d)=\theta$ and for some $\varepsilon >0$, $\theta(\varepsilon)=\{B_1,B_2\}$, $\ell(B_1,B_2)=R>\varepsilon$ and $\{B_1,B_2\}$ does not satisfy $P$ for $R$. 
\end{definicion}

\begin{obs} For any $HC$ method $\mathfrak{T}=\mathfrak{T}(\ell^{SL};P)$, $P$ is nontrivial for $\ell^{SL}$ if there exists some metric space $(X,d)$ with two $\varepsilon$-components $B_1,B_2$, such that $\ell^{SL}(B_1,B_2)=R>\varepsilon$ and $\{B_1,B_2\}$ does not satisfy $P$ for $R$. 
\end{obs}

Let $(X,d_0)$ be any metric space and $B_1,B_2$ a non-trivial partition of $X$ (this is,  $X=B_1\cup B_2$ with $B_1\cap B_2=\emptyset$ and $B_1,B_2\neq \emptyset$). Let $R>0$ and let 
\begin{equation}\label{Eq: Gamma}
 \Gamma^{B_1,B_2}_R\co [0,1]\to \mathcal{M} 
\end{equation}
be defined as follows:

$\Gamma^{B_1,B_2}_R(0)=(X,d_0)$, $\Gamma^{B_1,B_2}_R(1)=I(R)$ and  for every $t\in (0,1)$, $\Gamma(t)=(X,d_t)$ with

$d_t(x,y)=\left\{ 
\begin{array}{ccc}
(1-t)d_0(x,y) & \mbox{ if } x,y\in B_1 \mbox{ or } x,y\in B_2\\ 
(1-t)d_0(x,y)+tR & \mbox{ if } x\in B_1, y\in B_2 \mbox{ or } y\in B_1, x\in B_2 \end{array}
\right.$

\begin{prop} $\Gamma^{B_1,B_2}_R$ is well defined.
\end{prop}

\begin{proof} Let us see that $d_t$ is a metric for every $t\in (0,1)$. It suffices to check that the triangle inequality holds for every $t\in (0,1)$. Let $x,y,x\in X$. We know that $d(x,z)\leq d(x,y)+d(x,z)$. 

If $x,z$ are in the same block, $B_i$, then it is trivial to see that $d_t(x,z)=(1-t)d(x,z)\leq (1-t)[d(x,y)+d(y,z)]\leq d_t(x,y)+d_t(x,z)$. 

If $x,z$ are not in the same block, then $d_t(x,z)=(1-t)d(x,z)+tR\leq (1-t)[d(x,y)+d(y,z)]+tR = d_t(x,y)+d_t(x,z)$.

Hence, triangle inequality holds.
\end{proof}

\begin{prop}\label{Prop: cont} $\Gamma^{B_1,B_2}_R\co [0,1] \to  (\mathcal{M},d_{GH})$ is continuous (with the euclidean metric on $[0,1]$).
\end{prop}

\begin{proof} First, let us see that $\Gamma^{B_1,B_2}_R$ is continuous on $[0,1)$.

For any $s_1,s_2\in [0,1)$, consider the correspondence defined by the identity: $\tau=\{(x,x) \, | \, x\in X\}$. Then,  $|d_{s_1}(x,y)-d_{s_2}(x,y)|\leq |(s_2-s_1)d(x,y)|+|(s_2-s_1)R|=|s_2-s_1|(d(x,y)+R)$. 

Then, if $|s_2-s_1|<\gamma$, $d_{GH}((X,d_{s_1}),(X,d_{s_2}))\leq \frac{1}{2}|s_2-s_1|(d(x,y)+R)\leq \gamma \frac{diam(X,d)+R}{2}$. Thus, $\forall \, \delta >0$ and for any $\gamma\leq \frac{2\delta}{diam(X)+R}$, if $|s_2-s_1|<\gamma$, then $d_{GH}((X,d_{s_1}),(X,d_{s_2}))\leq \delta$.

To check the continuity of $\Gamma^{B_1,B_2}_R$ at $t=1$ consider the correspondence $\tau'=\{(B_1,\bar{p}_0)\cup (B_2,\bar{p}_1)\}$. Then, if $x,y$ are in the same block $B_i$,  $|d_t(x,y)-d_1(\bar{p}_i,\bar{p}_i)|= (1-t)d(x,y)\leq (1-t)diam(X)$. If $x,y$ are in different blocks, $|d_t(x,y)-d_1(\bar{p}_0,\bar{p}_1)|=|(1-t)d(x,y)+tR-R|\leq (1-t)[d(x,y)+R]\leq (1-t)[diam(X)+R]$. Thus, $d_{GH}((X,d_t),I(R))\leq (1-t)\frac{diam(X)+R}{2}$. Hence, for any $\delta>0$, if $1-\frac{2\delta}{diam(X)+R}<t<1$, then $d_{GH}((X,d_t),I_R)\leq \delta$.  

Therefore, $\Gamma^{B_1,B_2}_R$ is continuous.
\end{proof}

\begin{definicion} Let $\ell$ be a linkage function.  We say that a $\ell$ is \textbf{$\Gamma$-regular} if given any finite metric space $(X,d_0)$ and any non-trivial partition  of $X$, $B_1,B_2$, with $\ell(B_1,B_2)=R$, then $\ell_{R,d_t}(B_1,B_2)=R$ for every $t\in [0,1)$ where $\ell_{R,d_t}$ denotes the linkage function $\ell$ when it is applied to a pair of blocks in $\Gamma^{B_1,B_2}_R(t)=(X,d_t)$.
\end{definicion}

\begin{prop}\label{Prop: stable} $\ell^{SL}$, $\ell^{CL}$ and $\ell^{AL}$ are $\Gamma$-regular.
\end{prop}

\begin{proof} In the case of $\ell^{SL}$  it is immediate to check that the minimal distance, $R$, between $B_1$ and $B_2$ is constant on $\Gamma^{B_1,B_2}_R$ by construction.

The case of $\ell^{CL}$ is also trivial since the maximal distance, $R$, between $B_1$ and $B_2$ is constant on $\Gamma^{B_1,B_2}_R$ by construction.

Consider any nontrivial partition $B_1,B_2$ of $(X,d_0)$. Let $R=\frac{\sum_{(x,x')\in B_1\times B_2}d_0(x,x')}{\#(B_1)\cdot \#(B_2)}$. Now, for any $t\in [0,1]$, $$\ell_{R,d_t}(B_1,B_2)=\frac{\sum_{(x,x')\in B_1\times B_2}d_t(x,x')}{\#(B_1)\cdot \#(B_2)}=\frac{\sum_{(x,x')\in B_1\times B_2}[(1-t)d_0(x,x')+tR]}{\#(B_1)\cdot \#(B_2)}=$$
$$=(1-t)\frac{\sum_{(x,x')\in B_1\times B_2}d_0(x,x')}{\#(B_1)\cdot \#(B_2)}+tR=(1-t)R+tR=R.$$ 
\end{proof}

\begin{definicion} A linkage function  $\ell$ is \textbf{scale preserving} if for any space $X$, any pair of disjoint blocks, $A,B\subset X$, and any pair of metrics, $d,d'$, on $X$ such that for some $\alpha>0$, $d'(x,y)=\alpha\cdot d(x,y)$ $\forall \, x,y\in X$, then $\ell_{d'}(A,B)=\alpha\cdot \ell_d(A,B)$. 
\end{definicion}

\begin{obs}\label{Obs: scale} It is immediate to check that $\ell^{SL}$, $\ell^{CL}$ and $\ell^{AL}$ are scale preserving.
\end{obs}

\begin{lema}\label{Lema: scale} Let $\mathfrak{T}=\mathfrak{T}(\ell,\emptyset)$ with $\ell$ scale preserving. Let $(X,d)$ be a metric space, $\alpha>0$ and let $(X,d')$ be such that $d'(x,y):=\alpha \cdot d(x,y)$. Let $\mathfrak{T}_\mathcal{D}(X,d)=\theta$ and $\mathfrak{T}_\mathcal{D}(X,d')=\theta'$. Then, $\theta(t)=\theta'(\alpha \cdot t)$.
\end{lema}

\begin{proof} In the recursive formulation of $\mathfrak{T}$, let us denote by $\Theta_{R_{i}}$ the partitions built for $\theta$ and by $\Theta'_{R'_{i}}$ the partitions built for $\theta'$. Clearly, $\Theta_{0}=\Theta'_{0}$. Suppose that $R'_{i-1}=\alpha \cdot R_{i-1}$ and $\Theta_{R_{i-1}}=\Theta'_{R'_{i-1}}=\{B_1,...,B_k\}$. Since $\ell$ is scale preserving, it is immediate to check that $\ell_{d}(B_{j_1},B_{j_2})=\alpha\cdot \ell_{d'}(B_{j_1},B_{j_2})$ for every $1\leq j_1,j_2 \leq n$. Therefore, $R'_i=\alpha \cdot R_i$ and $\Theta_{R_{i}}=\Theta'_{\alpha\cdot R_{i}}$. Thus, by induction on $i$, it follows that $\theta(t)=\theta'(\alpha \cdot t)$  $\forall \, t\geq 0$.
\end{proof}

\begin{lema}\label{Lema: restriction} Let $\mathfrak{T}=\mathfrak{T}(\ell,\emptyset)$, let $(X,d)$ be a metric space and let 
$\mathfrak{T}_\mathcal{D}(X,d)=\theta$. Suppose $\theta(\varepsilon)=\{B_1,...,B_n\}$ and let $(X,d')$ be such that 
$d'(x,y)=d(x,y)$ if $x,y\in B_j$ for any $1\leq j \leq n$ and $d'(x,y)\geq d(x,y)$ for every $(x,y)\in B_i\times B_j$ with $i\neq j$. Then, if $\mathfrak{T}_\mathcal{D}(X,d')=\theta'$, $\theta'(t)=\theta(t)$ for every $t\leq \varepsilon$.
\end{lema}

\begin{proof} Since linkage functions are representation independent and monotonic, for every $A,B\subset B_j$, 
$\ell_{d}(A,B)=\ell_{d'}(A,B)$ and for every $(A,B)\subset B_i\times B_j$, $\ell_{d}(A,B)\leq \ell_{d'}(A,B)$.

Then, let us apply induction on the construction of the dendrograms. Clearly, $R_0=0=R'_0$ and $\Theta_1=\Theta'_1$.
Suppose $R'_0=R_0,...,R'_{i-1}=R_{i-1}$ and  $\Theta'_{i}=\Theta_{i}=\{C_1,...,C_k\}$. Then, since 
$\theta(\varepsilon)=\{B_1,...,B_n\}$, if  $\ell_d(C_r,C_s)=\min\{\ell(C,C')\, | \, C,C'\in \Theta_i, \ C\neq C'\}\leq \varepsilon$, 
$C,C'$ are both contained in the same block $B_j$. Then, $\ell_{d'}(C_r,C_s)=\ell_d(C_r,C_s)$, $R'_i=R_i$ and 
$\Theta'_{i+1}=\Theta_{i+1}$. Therefore, $\theta'(t)=\theta(t)$ for every $t\leq \varepsilon$.  
\end{proof}

\begin{definicion} Let $P$ be an unchaining condition and $\ell$ be a linkage function. We say that $P$ is \textbf{consistent for $\ell$} if the following implication holds.

Let $(X,d)$ be any metric space  and $B_1,B_2$ any nontrivial partition  of $X$ with $\ell_d(B_1,B_2)=R$ and such that $\{B_1,B_2\}$ does not satisfy $P$ for $R$ in $(X,d)$.

Let $(X,d')$ be such that $d'(x,y):=d(x,y)$ for every $x,y\in B_1$ and for every $x,y\in B_2$, and 
$d'(x,y)>d(x,y)$ for every $x,y\in B_1\times B_2$. 

Then, if $\ell_{d'}(B_1,B_2)=R'$, $B_1,B_2$ does not satisfy $P$ for $R'$ in $(X,d')$.   
\end{definicion}

\begin{lema}\label{Lema: minimal} Let $\mathfrak{T}(\ell;P)$ be an almost-standard linkage-based  $HC$ method such that $P$ is nontrivial and consistent for $\ell$. Let $(X,d)$ be a metric space such that, for $\mathfrak{T}_\mathcal{D}(X,d)=\theta$, $\theta(\varepsilon)=\{B_1,B_2\}$, 
$\ell_d(B_1,B_2)=R>\varepsilon$ and $\{B_1,B_2\}$ does not satisfy $P$ for $R$. Suppose that $\#(X)$ is the minimal cardinal for which there exists such a metric space. Then, for any metric space $(X',d')$ with $\#(X')<\#(X)$, if 
$\mathfrak{T}_\mathcal{D}(X',d')=\theta'$ and $\theta'(\varepsilon')=\{B'_1,B'_2\}$ with $\ell_{d'}(B'_1,B'_2)=R'$, then 
$\{B'_1,B'_2\}$ satisfies $P$ for $R'$.
\end{lema}

\begin{proof} Since $\#(X)$ is minimal, if $R'>\varepsilon$ we are done. 

Let $R'\leq \varepsilon$. Suppose that $\{B'_1,B'_2\}$ does not satisfy $P$ for $R'$ and suppose that $\#(X')$ is also minimal so that there is a pair of blocks which do not satisfy $P$. By the properties of linkage functions, there is a metric $d''$ such that it extends $d'$ on $B'_1$ and $B'_2$ and such that $\ell_{d''}(B'_1,B'_2)=R''>\varepsilon'$. 
Since $\#(X')$ is minimal, $P$ does not apply in the construction of the dendrogram $\theta'(t)$ for $t\leq \varepsilon$. Therefore, if $\mathfrak{T}_\mathcal{D}(X',d'')=\theta''$, then by Lemma \ref{Lema: restriction}, $\theta''(\varepsilon')=\{B'_1,B'_2\}$. Now,
since $P$ is consistent for $\ell$, $\{B'_1,B'_2\}$ does not satisfy $P$ for $R''>\varepsilon'$ which contradicts 
the minimality of $\#(X)$.
\end{proof}

\begin{definicion} We say that an almost-standard linkage-based  $HC$ method $\mathfrak{T}(\ell;P)$ is \textbf{admissible} if the following conditions hold:
\begin{itemize} 
	\item $\ell$ is $\Gamma$-regular and scale preserving, 
	\item $P$ is nontrivial and consistent for $\ell$,  
	\item $\mathfrak{T}_\mathcal{U}(I(R))=I(R)$. 
\end{itemize}
\end{definicion}

\begin{obs} Let $\mathfrak{T}(\ell;P)$ be an almost-standard linkage-based $HC$ method such that $\ell$ is $\ell^{SL}$, $\ell^{CL}$ or $\ell^{AL}$. Then, by \ref{Prop: stable} and \ref{Obs: scale}, $\mathfrak{T}$ is admissible if $P$ is nontrivial and consistent for 
$\ell$ and $\mathfrak{T}_\mathcal{U}(I(R))=I(R)$. 
\end{obs}

\begin{prop}\label{Prop: main_1} Let $\mathfrak{T}=\mathfrak{T}(\ell;P)$ be any admissible almost-standard linkage-based $HC$ method. Then, there are constants $\varepsilon,R\in \br$ such that for every $\delta>0$ there exists a pair metric spaces $(Z,d),(Z',d')$ holding the following conditions:
\begin{itemize}
	\item $\theta(\varepsilon)=\{C_1,C_2\}$, where $\theta=\mathfrak{T}_\mathcal{D}(Z,d)$, 
	\item $\theta'(\varepsilon)=\{C'_1,C'_2\}$, where $\theta'=\mathfrak{T}_\mathcal{D}(Z',d')$,
	\item $\ell(C_1,C_2)=\ell(C'_1,C'_2)=R>\varepsilon$
	\item $d_{GH}((Z,d),(Z',d'))<\delta$,
	\item $\{C_1,C_2\}$ does not satisfy $P$ for $R$ i.e. $C_1\not \sim_R C_2$
	\item $\{C'_1,C'_2\}$ satisfies $P$ for $R$ i.e. $C'_1\sim_R C'_2$
\end{itemize}
\end{prop}

\begin{proof} Since condition $P$ is nontrivial for $\ell$, there exists some metric space $(X,d_0)$ such that, for $\mathfrak{T}_\mathcal{D}(X,d_0)=\theta_0$, $\theta_0(\varepsilon)=\{B_1,B_2\}$, $\ell(B_1,B_2)=R>\varepsilon$ and $\{B_1,B_2\}$ does not satisfy $P$ for $R$. Suppose that, given $P$, $\#(X)$ is the minimal cardinal for which there exists a metric space holding these conditions.

Let $\Gamma:=\Gamma^{B_1,B_2}_R \co [0,1]\to \mathcal{M}$ be defined as in (\ref{Eq: Gamma}).

As we proved in Proposition \ref{Prop: cont}, $\Gamma$ is a path in $\mathcal{M}$.

Now, let us see that $\Gamma$ is a path in $\mathcal{H}(\mathfrak{T},\varepsilon)$ from $(X,d_0)$ to $I(R)$. Since $\ell$ is $\Gamma$-regular, $\ell_{R,d_t}(B_1,B_2)=R>\varepsilon$.

Let us see that $\theta_{t}(\varepsilon)=\{B_1,B_2\}$ for every $t\in (0,1)$.

First, notice that since $\#(X)$ is minimal, by Lemma \ref{Lema: minimal}, if we apply the algorithm on $X$ with any metric, every pair of blocks satisfies $P$ as long as the pair of blocks involved do not include all the points in $X$.

Let $(X,d'):=(X,(1-t)d_0)$. Since $\ell$ is scale preserving, $\theta'((1-t)r)=\theta(r)$ where 
$\theta'=\mathfrak{T}_\mathcal{D}(X,d')$. In particular, $\theta'((1-t)\varepsilon)=\{B_1,B_2\}$.

Then, let us see that $\theta_t(r)=\theta'(r)$ for every $r\leq (1-t)\varepsilon$ where $\theta_{X_t}=\mathfrak{T}_\mathcal{D}(X,d_t)$. Let $\Theta_i$ denote the partitions defined in the construction of $\theta_t$ and $\Theta'_i$ denote the partitions defined in the construction of $\theta'$. Clearly, $\Theta_0=\Theta'_0$. Suppose $\Theta_{i-1}=\Theta'_{i-1}=\{V_1,...,V_n\}$ with $V_1\cup\cdots V_k=B_1$ and $V_{k+1}\cup\cdots V_n=B_2$. Let $R'_{i-1}:=\min\{\ell_{d'}(V,V')\, | \, V,V'\in \Theta'_i, \ V\neq V'\}$ and assume $R'_{i-1}\leq (1-t)\varepsilon$. Notice that since $\theta'((1-t)\varepsilon)=\{B_1,B_2\}$, if $\ell_{d'}(V,V')=R'_{i-1}$, then $V,V'$ are in the same block, $B_1$ or $B_2$. Now, since $\ell$ is monotonic (by definition), $\ell_{d_t}(V,V')=\ell_{d'}(V,V')$ if $V,V'$ are in the same block, $B_1$ or $B_2$, and $\ell_{d_t}(V,V')\geq \ell_{d'}(V,V')$ if $V\subset B_1$ and $V'\subset B_2$. Therefore, $R_{i-1}=R'_{i-1}$ and $\Theta_{i}=\Theta'_{i}$. Hence, it follows that $\theta_t(r)=\theta'(r)$ for every $r\leq (1-t)\varepsilon$. In particular, $\theta_t((1-t)\varepsilon)=\{B_1,B_2\}$.

Thus, since $\ell_{R,d_t}(B_1,B_2)=R>\varepsilon$,  $\theta_t(\varepsilon)=\{B_1,B_2\}$ for every $t\in (0,1)$.

Since $\mathfrak{T}$ is admissible, if $\mathfrak{T}_{\mathcal{D}}(I_R)=\theta_1$, $\theta_1(t)$ has two blocks $\{\bar{p}_0\},\{\bar{p}_1\}$ for every $t<R$.  

Then, $\Gamma(t)\in \mathcal{H}(\mathfrak{T},\varepsilon)$ for every $t\in [0,1]$.

Since $\Gamma$ is a path in $\mathcal{H}(\mathfrak{T},\varepsilon)$, for every $\gamma$ there exist $s_1,s_2\in [0,1]$ with $|s_2-s_1|<\gamma$ such that  $\gamma(s_1)=(X,d_{s_1})$ has two blocks $C_1:=B_1,C_2:=B_2$ with $C_1\not \sim_R C_2$ and $\Gamma(s_2)=(X,d_{s_2})$ has two blocks $C'_1:=B_1,C'_2:=B_2$ with $C'_1\sim_R C'_2$. 

Finally, as we saw in the proof of Proposition \ref{Prop: cont}, taking $\gamma<\frac{2\delta}{diam(X,d)+R}$, $d_{GH}((X,d_{s_1}),(X,d_{s_2}))<\delta$.  
\end{proof}

\begin{definicion} We say that a pair of $t$-components $B_1,B_2$ in $(X,d)$ are \textbf{$(t,R)$-bridged by a single edge} if there exist $b_1\in B_1$, $b_2\in B_2$ such that $d(b_1,b_2)=R>t$ and for every $(b_1,b_2)\neq (x,y) \in B_1\times B_2$, $d(x,y)>R$. If the constants $t,R$ are not relevant we say that $B_1,B_2$  are \textbf{bridged by a single edge}.  
\end{definicion}

\begin{definicion} Given $\mathfrak{T}(\ell^{SL},P)$, we say that $P$ is \textbf{bridge-unchaining} 
if there is a metric space $(X,d)$ and a nontrivial partition $B_1,B_2$ such that  $B_1,B_2$ are bridged by a single edge and
$\{B_1,B_2\}$ does not satisfy $\{P\}$ for $R$.
\end{definicion}

\begin{definicion} Given $\mathfrak{T}(\ell^{SL},P)$, we say that $P$ is \textbf{minimally bridge-unchaining} if: 
\begin{itemize}
	\item there is a metric space $(X,d)$ and a nontrivial partition $B_1,B_2$ such that  $B_1,B_2$ are bridged by a single edge and 
$\{B_1,B_2\}$ does not satisfy $\{P\}$ for $R$, 
	\item $(X,d)$ can be chosen so that for every metric space $X'$ with $\#(X')<\#(X)$ and any nontrivial partition $B'_1,B'_2$ of 
$X'$ with $\ell^{SL}(B_1,B_2)=\varepsilon$, then $\{B'_1,B'_2\}$ satisfies $\{P\}$ for $\varepsilon$.
\end{itemize}
\end{definicion}

\begin{prop}\label{Prop: bridge} For every $\alpha\geq 1$, $P_\alpha$ (the unchaining condition of $SL(\alpha)$) is minimally bridge-unchaining.
\end{prop}

\begin{proof} Consider $X=B_1\cup B_2$ as follows. Let $B_1:=\{x_0,...,x_{\alpha+1}\}$ with $d(x_i,x_j)=1$ for every $i\neq j$ and let $B_2:=\{y_0,...,y_{\alpha+1}\}$ 
with $d(y_i,y_j)=1$ for every $i\neq j$. Suppose $d(x_0,y_0)=2$ and $d(x_i,y_j)>2$ for every $(i,j)\neq (0,0)$.

It is immediate to check that $B_1,B_2$ are $(1,2)$-bridged by a single edge which is $(x_0,y_0)$. 

Now, $F_2(X)$ has a unique simplex, $\Delta=[x_0,y_0]$, intersecting $B_1$ and $B_2$ and $dim(\Delta)=1$. Nevertheless, $dim(F_2(B_1))=dim(F_2(B_2))=\alpha+1$. Therefore, $B_1,B_2$ does not satisfy $P_\alpha$ for $2$.

Notice that $\#(X)=2\alpha +4$. Suppose $X'=B'_1\cup B'_2$ such that $B'_1,B'_2$ does not satisfy $P_\alpha$ for some $R$. Then, in particular, $\min\{dim(F_R(B_1)),dim(F_R(B_2))\}>\alpha$. Hence, $\#(B_1),\#(B_2)>\alpha+1$ and $\#(X')\geq \#(X)$. Thus, $P_\alpha$ is minimally bridge unchaining.
\end{proof}

\begin{definicion} We say that a $HC$ method $\mathfrak{T}(\ell;P)$ is \textbf{ordinary} if $\mathfrak{T}_\mathcal{U}(I(R))=I(R)$, 
$\mathfrak{T}_\mathcal{U}(I(\delta,R))=\mathfrak{T}_\mathcal{U}^{SL}(I(\delta,R))$ and $\mathfrak{T}_\mathcal{U}(I(\delta,R,\delta))=\mathfrak{T}_\mathcal{U}^{SL}(I(\delta,R,\delta))$ for every $\delta,R>0$.
\end{definicion}

\begin{obs} For every $\alpha\geq 1$,  $SL(\alpha)$ is ordinary.
\end{obs}

\begin{teorema}\label{Teorema: stable} If $\mathfrak{T}(\ell^{SL};P)$ is an ordinary almost-standard linkage-based $HC$ method and $P$ is minimally bridge-unchaining, then $\mathfrak{T}$ is not stable in the Gromov-Hausdorff sense.
\end{teorema}

\begin{proof} Since $P$ is bridge-unchaining, there exists some metric space $(X,d)$ with two $\varepsilon$-components $B_1,B_2$ bridged by a single edge $(b_1,b_2)$ with  $b_1\in B_1$, $b_2\in B_2$, $d(b_1,b_2)=R>\varepsilon$ and 
such that $\{B_1,B_2\}$ does not satisfy $P$ for $R$. Suppose that $\#(X)$ is the minimal cardinal for which there is such a pair.

Let $R=t_i$ for some $1\leq i \leq n$. Then $\varepsilon\leq t_{i-1}$, $B_1,B_2\in \theta(t_{i-1})$ and $B_1\not \sim_R B_2$. 

Since $\mathfrak{T}$ is ordinary, either $B_1$ or $B_2$ has at least two points. Let us assume that $\#(B_1)>1$.

Let $\delta =\min_{1\leq i \leq n}\{t_i-t_{i-1}\}$ (in particular, $\delta\leq t_1=min_{x\neq x'}d(x,x')$).

If  $\#(B_2)=1$, let  $\Gamma\co [0,1]\to (\mathcal{M},d_{GH})$ be such that $\Gamma(0)=(X,d)$, $\Gamma(1)=I(\delta,R)$ and for every $t\in (0,1)$, $\Gamma(t)=(X,d_t)$ with

$d_t(x,y)=\left\{ 
\begin{tabular}{l}
$(1-t)d(x,y)  \qquad \mbox{ if } x,y\in B_1\backslash\{b_1\}  $\\
$(1-t)d(x,y) +t \delta \quad \mbox{ if }  x=b_1, \, y\in B_1\backslash\{b_1\} \mbox{ or } y=b_1, \, x\in B_1\backslash\{b_1\} $ \\ 
$(1-t)d(x,y)+t(R+\delta)  \  \mbox{ if } x\in B_1\backslash \{b_1\}, y=b_2 \mbox{ or } y\in B_1\backslash \{b_1\}, x=b_2 $\\
$\qquad R \qquad \qquad  \mbox{ if } x=b_1, y=b_2 \mbox{ or } y=b_1, x=b_2 $  \end{tabular}
\right.$

If $\#(B_2)>1$, let $\Gamma\co [0,1]\to (\mathcal{M},d_{GH})$ be such that $\Gamma(0)=(X,d)$, $\Gamma(1)=I(\delta,R,\delta)$  and for every $t\in (0,1)$, $\Gamma(t)=(X,d_t)$ with $d_t(x,y)=$

$=\left\{ 
\begin{tabular}{l}
$(1-t)d(x,y) \qquad  \qquad \quad \ \ \mbox{ if } x,y\in B_1\backslash\{b_1\} \mbox{ or } x,y\in B_2\backslash\{b_2\} \qquad \qquad$\\
$(1-t)d(x,y) +t \delta \qquad \quad \ \ \mbox{ if }  x=b_1, \, y\in B_1\backslash\{b_1\} \mbox{ or } y=b_1, \, x\in B_1\backslash\{b_1\}  $\\ 
$(1-t)d(x,y) +t \delta \qquad \quad \ \ \mbox{ if }  x=b_2, \, y\in B_2\backslash\{b_2\} \mbox{ or } y=b_2, \, x\in B_2\backslash\{b_2\}  $\\ 
$(1-t)d(x,y)+t(R+\delta) \quad  \mbox{ if } x\in B_1\backslash \{b_1\}, y=b_2 \mbox{ or } y\in B_1\backslash \{b_1\}, x=b_2 $\\
$(1-t)d(x,y)+t(R+\delta) \quad  \mbox{ if }  x=b_1, y\in B_2\backslash \{b_2\}  \mbox{ or } y=b_1, x\in B_2\backslash \{b_2\}$\\
$(1-t)d(x,y)+t(R+2\delta) \ \, \mbox{ if } x\in B_1\backslash \{b_1\}, y\in B_2\backslash \{b_2\} \mbox{ or } y\in B_1\backslash \{b_1\}, x\in B_2\backslash \{b_2\}$\\
$\qquad \ \ R \qquad \qquad \, \qquad \qquad  \mbox{ if } x=b_1, y=b_2 \mbox{ or } y=b_1, x=b_2$ \end{tabular}
\right.$

\begin{figure}[ht]
\centering
\includegraphics[scale=0.5]{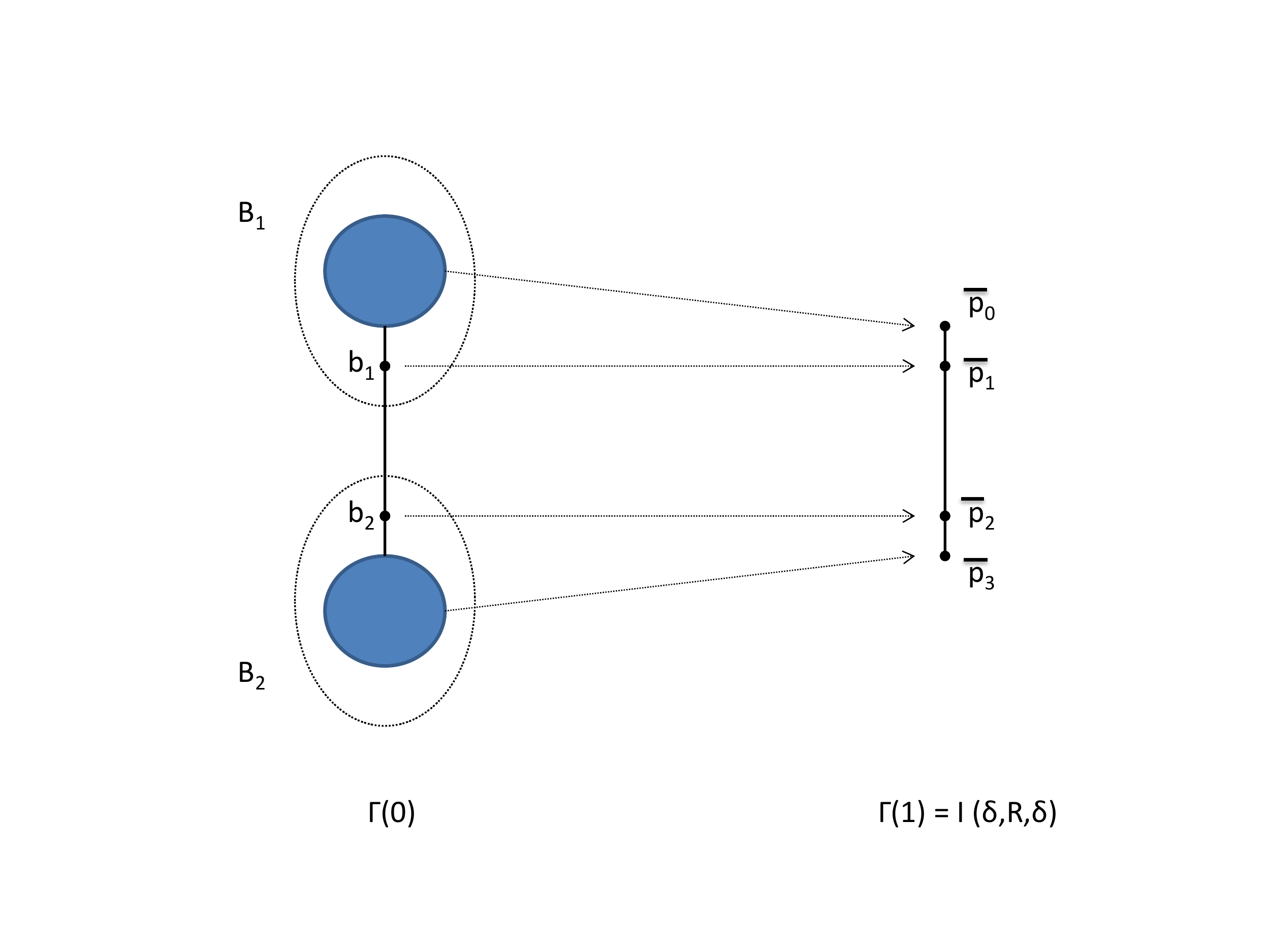}
\caption{$\Gamma$ defines a path in $\mathcal{H}(\mathfrak{T},\varepsilon)$ from $(X,d)$ to $I(\delta,R)$ or $I(\delta,R,\delta)$.}
\label{Fig: Path}
\end{figure}

In either case, let us see that $d_t$ is a metric for every $t\in [0,1)$. It suffices to check that the triangle inequality holds for every $t\in [0,1)$.

First, notice that given $x,y,z\in X$, since $d$ is a metric, for every $0\leq t\leq 1$, $$(1-t)d(x,z)\leq (1-t)d(x,y)+(1-t)d(y,z).$$

Consider the following cases:
\begin{itemize}
	\item If $x,z$ are in the same $t$-component $B_i$. 
	
Then, if $b_i\not \in \{x,z\}$, $d_t(x,z)=(1-t)d(x,z)\leq (1-t)d(x,y)+(1-t)d(y,z)\leq d_t(x,y)+d_t(y,z)$. 

If $b_i \in \{x,z\}$, $d_t(x,z)=(1-t)d(x,z)+t\delta \leq (1-t)d(x,y)+(1-t)d(y,z)+t\delta= d_t(x,y)+d_t(y,z)$.
	
	\item If $x,z$ are not in the same $t$-component. 
	
Then, if $ \{x,z\}=\{b_1,b_2\}$, $d_t(x,z)=(1-t)d(x,z)+tR\leq (1-t)d(x,y)+(1-t)d(y,z)+tR< d_t(x,y)+d_t(y,z)$. 

If $ \{x,z\}\cap \{b_1,b_2\}=b_i$, $i=1,2$, $d_t(x,z)=(1-t)d(x,z)+tR+ t\delta \leq (1-t)d(x,y)+(1-t)d(y,z)+tR+t\delta \leq  d_t(x,y)+d_t(y,z)$. 

If $ \{x,z\}\cap \{b_1,b_2\}=\emptyset$ $d_t(x,z)=(1-t)d(x,z)+tR+ 2t\delta \leq (1-t)d(x,y)+(1-t)d(y,z)+tR+2t\delta \leq  d_t(x,y)+d_t(y,z)$.
\end{itemize}

Claim 1. $\Gamma$ is a path in $(\mathcal{M},d_{GH})$ 
from $(X,d)$ to  $I(\delta,R)$ or $I(\delta,R,\delta)$ respectively.

Consider for every $t_1<t_2<1$ the correspondence induced by the identity, $\tau=\{(x,x) \, | \, x\in X\}$. Then, for every pair of points $x,y\in X$, $|d_{t_1}(x,y)-d_{t_2}(x,y)|\leq (t_2-t_1)[d(x,y)+R+2\delta]$. Hence, $\Gamma$ is continuous on $[0,1)$. 

To check the continuity at $t=1$, consider the correspondence $\tau'$ given by $(B_1\backslash \{b_1\},\bar{p}_0)$ , 
$(b_1, \bar{p}_1)$, and $(b_2, \bar{p}_r)$ if $\#(B_2)=1$ or $(B_1\backslash \{b_1\},\bar{p}_0)$, $(b_1, \bar{p}_1)$,  
$(p_2, \bar{p}_2)$ and $(B_2\backslash \{b_2\},\bar{p}_3)$ if $\#(B_2)>1$. 

Then, given $(x,a),(y,b)\in \tau'$, in every case $|d_{t}(x,y)-d_{1}(a,b)|\leq (1-t)[d(x,y)+R+2\delta]$ and $\Gamma$ is continuous.

Claim 2. In either case, $\Gamma(t)\in \mathcal{H}(\mathfrak{T},\varepsilon)$ for every $t\in [0,1]$. We assumed that $\#(X)$ is the minimal cardinal for which there is a pair of blocks bridged by a single edge which does not satisfy $P$. Since $P$ is minimally bridge-unchaining, this implies that $\#(X)$ is the minimal cardinal for which there is a pair of blocks which does not satisfy $P$. Hence, we may assume that $P$ is always satisfied as long as the blocks involved are not a partition of $X$. 
$\mathfrak{T}=\mathfrak{T}^SL$.

\begin{quote} Notice that this is the unique step where we need the assumption of being \emph{minimally} bridge unchaining. Therefore, this theorem can be rewritten asking $P$ to be bridge unchaining and any other condition that allows us to assume that condition $P$ is satisfied in the construction of the dendrogram for every $R_i\leq \varepsilon$. 
\end{quote}

Also, since $\delta\leq t_1$, it is immediate to check that for every $t\in (0,1)$, $d_t(x,y)\leq d(x,y)$ for every $x,y\in B_r$ with $r=1,2$. Thus, $B_1,B_2$ are $\varepsilon$-connected in $(X,d_t)$.  Then, if 
$\mathfrak{T}_\mathcal{D}(X,d_t)=\theta_t$, and $P$ is always satisfied for every $t\leq \varepsilon$, 
$\theta_t(\varepsilon)=\{B_1,B_2\}$.

Claim 3. For every $t\in [0,1)$ there is no pair of points $x\in B_1$, $y\in B_2$  with $R< d_t(x,y) < R+\delta$.  Notice that if $R=t_i$, $t_{i+1}\geq R+\delta$ and for every $x\in B_1$, $y\in B_2$, then either $x=c_1$ and $y=c_2$ with $d_t(x,y)=R$, or $d(x,y)\geq t_{i+1}\geq R+\delta$ and $d_t(x,y)\geq (1-t)d(x,y)+t(R+\delta) \geq R+\delta$.

Since $\mathfrak{T}(\ell^{SL};P)$ is an ordinary $HC$ method, $\mathfrak{T}(I(\delta,R))=\mathfrak{T}_{SL}(I(\delta,R))$ and 
$\mathfrak{T}(I(\delta,R,\delta))=\mathfrak{T}_{SL}(I(\delta,R,\delta))$. Since $\#(X)$ is the minimal cardinal for which condition 
$P$ has an effect, $\theta_{t}(\varepsilon)=\{B_1,B_2\}$ for every $t\in [0,1)$. 
Since $\Gamma$ is continuous, for any $\delta_1>0$ there exist $s_1,s_2\in [0,1]$ with $|s_2-s_1|<\delta_1$ such that $\Gamma(s_1)=(X,d_{s_1})$ has two $\varepsilon$-components $B_1,B_2$ with 
$B_1\not \sim_R B_2$ and $\Gamma(s_2)=(X,d_{s_2})$ has two $\varepsilon$-components $B'_1,B'_2$ with $B'_1\sim_R B'_2$. If $s_2<1$ let $B'_1:=B_1$ and $B'_2:=B_2$. If $s_2=1$ and $\Gamma(1)=I(\delta,R)$ let $B'_1:=\{\bar{p}_0,\bar{p}_1\}$ and $B'_2:=\{\bar{p}_2\}$. If $s_2=1$ and $\Gamma(1)=I(\delta,R,\delta)$ let $B'_1:=\{\bar{p}_0,\bar{p}_1\}$ and $B'_2:=\{\bar{p}_2,\bar{p}_3\}$.

Let $\mathfrak{T}((X,d_{s_1}))=(X,u_{s_1})$ and $\mathfrak{T}((X,d_{s_2}))=(X,u_{s_2})$. For every $x\in B_1$, $y\in B_2$, 
$u_{s_2}(x,y)=R$ and, since  $d_{s}(x,y)\not \in (R, R+\delta)$ for any $s\in [0,1]$, $u_{s_1}(x,y)\geq R+\delta$. 

Claim 4. $d_{GH}((X,u_{s_1}),(X,u_{s_2}))\geq \frac{\delta}{2} $. 

Suppose there is a correspondence $\tau\subset X\times X$ such that 
\begin{equation}\label{Correspondence} \sup_{(x,y),(x',y')\in \tau}|u_{s_1}(x,x')-u_{s_2}(y,y')|<\delta .\end{equation}

First, let us check that if $(x,y),(x',y')\in \tau$ and $x,x'$ are in the same $\varepsilon$-component of $(X,d)$, then $y,y'$ are also in the same $\varepsilon$-component of $(X,d)$. Let $x=x_0,...,x_n=x'$ be a $\varepsilon$-chain in $(X,d)$. Then, it is a $\varepsilon$-chain in $(X,d_{s_1})$ and, since $\#(C_1)<\#(X)$, condition $P$ does not apply. Hence ($B_1$ is a block in $\theta_{s_1}(\varepsilon)$ and) $u_{s_1}(x,x')\leq \varepsilon=t_{i-1}$. If $y,y'$ are not in the same $\varepsilon$-component, then $u^{s_1}_{SL}(y,y'),u^{s_2}_{SL}(y,y')\geq R=t_i$. Therefore, $|u_{s_1}(x,x')-u_{s_2}(y,y')|\geq t_i-t_{i-1}\geq \delta$ which contradicts (\ref{Correspondence}).

Let $(b_1,y)(b_2,y')\in \tau$. Let us recall that $|u_{s_1}(c_1,c_2)|=R$ and $|u_{s_2}(y,y')|\geq R+\delta$ since $y,y'$ are in different $\varepsilon$-components.  Hence, $|u_{s_1}(b_1,b_2)-u_{s_2}(y,y')|\geq \delta$ which contradicts (\ref{Correspondence}).

Hence, $\sup_{(x,y),(x',y')\in \tau}|u_{s_1}(x,x')-u_{s_2}(y,y')|\geq \delta$ for every correspondence $\tau$, this is, 
$d_{GH}((X,u_{s_1}),(X,u_{s_2}))\geq \frac{\delta}{2} $. Since $\Gamma$ is continuous and $s_2-s_1$ is arbitrarily small, then 
$d_{GH}((X,d_{t_1}),(X,d_{t_2}))$ is also arbitrarily small. Therefore, $\mathfrak{T}$ is not stable.
\end{proof}

From Proposition \ref{Prop: bridge} and Theorem \ref{Teorema: stable},

\begin{cor} $SL(\alpha)$ is not stable in the Gromov-Hausdorff sense.
\end{cor}

\section{Discussion}\label{Section: Conclusions}

Using some kind of metric between metric spaces and how the distance between the input determines the distance between the output is a natural way of measuring stability. One of the advantages of using Gromov-Hausdorff metric is that it allows us 
to measure the distance between metric spaces which do not have the same number of points and this is necessary to check the stability of a method when the input has no fixed number of points. However, we prove that if the input does not correspond to an ultrametric space, Gromov-Hausdorff stability is a rare property.

Among the typical linkage-based methods, only single linkage is stable. If we introduce an unchaining property, then we may say that surely the method is going to be unstable. 

We believe that almost-standard linkage-based algorithms can be used to treat different problems related with the chaining effect. Stability is, of course, a strongly desirable property for the method. From our results, we conclude that Gromov-Hausdorff metric is not the appropriate metric to measure the stability of an almost-standard linkage-based algorithm.  Other possible metrics will be analyzed in future research.

\end{document}